\documentclass[letterpaper, 10 pt, conference]{ieeeconf}

\usepackage{microtype}
\usepackage{graphicx}
\usepackage{booktabs} % for professional tables

\usepackage[utf8]{inputenc} % allow utf-8 input
\usepackage[T1]{fontenc}    % use 8-bit T1 fonts
\usepackage{hyperref}       % hyperlinks
\usepackage{url}            % simple URL typesetting
\usepackage{booktabs}       % professional-quality tables
\usepackage{amsfonts}       % blackboard math symbols
\usepackage{nicefrac}       % compact symbols for 1/2, etc.
\usepackage{microtype}      % microtypography
\usepackage{amsmath,bm}
\usepackage{lipsum}
\usepackage{color}
\usepackage{subcaption}
\usepackage{caption}% http://ctan.org/pkg/caption
\usepackage{multirow}
\usepackage{dblfloatfix}
\usepackage{algorithmic}

\usepackage{mathtools}
\DeclarePairedDelimiter\abs{\lvert}{\rvert}%abs
\DeclarePairedDelimiter\norm{\lVert}{\rVert}%norm

\DeclareMathOperator*{\argmin}{arg\,min}

\usepackage{color}
\usepackage[ruled, linesnumbered]{algorithm2e}
\newlength\mylen

\usepackage{amsthm,amssymb}
\newtheorem{mydef}{Definition}
\newtheorem{lemma}{Lemma}
\newtheorem{theorem}{Theorem}
\newtheorem{assumption}{Assumption}

\newtheorem{innercustomthm}{Theorem}
\newenvironment{customthm}[1]
  {\renewcommand\theinnercustomthm{#1}\innercustomthm}
  {\endinnercustomthm}

\usepackage{hyperref}
\newcommand{\sgn}{\text{sgn}}

\IEEEoverridecommandlockouts                              % This command is only needed if 
\title{\LARGE \bf
Learning from Demonstration without Demonstrations
}

\author{Tom Blau$^{*,\dagger,\diamond}$, Philippe Morere$^{\dagger}$, Gilad Francis$^{\dagger}$% <-this % stops a space
\thanks{* Correspondence to: Tom Blau, {\tt\small tom.blau@data61.csiro.au}}%
\thanks{$^{\dagger}$ School of Computer Science, The University of Sydney, Australia}%
\thanks{$^{\diamond}$ CSIRO, Australia}%
}

\begin{document}

\maketitle

\begin{abstract}
State-of-the-art reinforcement learning (RL) algorithms suffer from high sample complexity, particularly in the sparse reward case.
A popular strategy for mitigating this problem is to learn control policies by imitating a set of expert demonstrations. The drawback of such approaches is that an expert needs to produce demonstrations, which may be costly in practice.
To address this shortcoming, we propose Probabilistic Planning for Demonstration Discovery (P2D2), a technique for automatically discovering demonstrations without access to an expert.
We formulate discovering demonstrations as a search problem and leverage widely-used planning algorithms such as Rapidly-exploring Random Tree to find demonstration trajectories. These demonstrations are used to initialize a policy, then refined by a generic RL algorithm.
We provide theoretical guarantees of P2D2 finding successful trajectories, as well as bounds for its sampling complexity.
We experimentally demonstrate the method outperforms classic and intrinsic exploration RL techniques in a range of classic control and robotics tasks, requiring only a fraction of exploration samples and achieving better asymptotic performance.
\end{abstract}

\setlength{\textfloatsep}{21pt}

\section{Introduction}
\label{sec:intro}
Reinforcement Learning (RL) studies how agents can learn a desired behaviour by simply using interactions with an environment and a return signal. 
% Exploration with sparse rewards sucks, leading to random walk behaviours
While RL algorithms have enjoyed many recent successes in fields ranging from game playing to robotic control~\cite{mnih2015human,tan18sim}, a high computational cost is needed to find efficient control policies.
This issue is particularly severe when rewards are sparse, indicating only binary success ($1$) or failure ($0$ or $-1$). Since most rewards are identical, there is little gradient information to guide policy learning. 
Figure~\ref{fig:glorot_initialization} offers some insight into this phenomenon: exploration in regions where the return surface is flat leads to a random walk type search. This inefficient search continues until non-zero gradients are found, which can then be followed to a local optimum.

However, if the policy can be initialized to a region of parameter space where the gradient is informative (i.e. not flat), reinforcement learning can optimize the policy efficiently, as shown in Figure~\ref{fig:handpicked_initialization}. Learning from demonstration (LfD) is a popular class of algorithms that can achieve such initialization. In LfD, the training data is a set of demonstrations from an expert policy that exhibits the desired agent behaviour. Control policies can be learned from this data either by pure supervised learning~\cite{pomerleau1991efficient}, supervised learning followed by RL refinement~\cite{schaal1997learning}, or interleaving RL and supervision~\cite{ho2016generative}.
The drawback of LfD approaches is that they require expert demonstrations. The availability of such demonstrations is taken for granted in much of the LfD literature, and there is an implicit assumption that the cost of acquiring demonstrations is negligible. However, for many problems this acquisition is difficult and costly, possibly defeating the purpose of using LfD in the first place.

To address this issue, we propose to formulate the acquisition of demonstrations as a search problem in state space. The demonstration finding problem is then solved by planning algorithms, yielding a method for acquiring demonstrations automatically (without an external expert). Planning algorithms can achieve much better exploration performance than the random walk of Figure~\ref{fig:glorot_initialization} by taking search history into account~\cite{Lavalle98rapidly-exploringrandom}. These techniques are also often guaranteed to find a solution in finite time if one exists~\cite{Karaman2011}. Demonstrations found by planning algorithms are then used to pre-train RL policies, initializing them in regions of parameter space where the return gradient is informative.

\begin{figure*}
\centering
\begin{subfigure}{.5\textwidth}
  \centering
  \includegraphics[width=.85\linewidth]{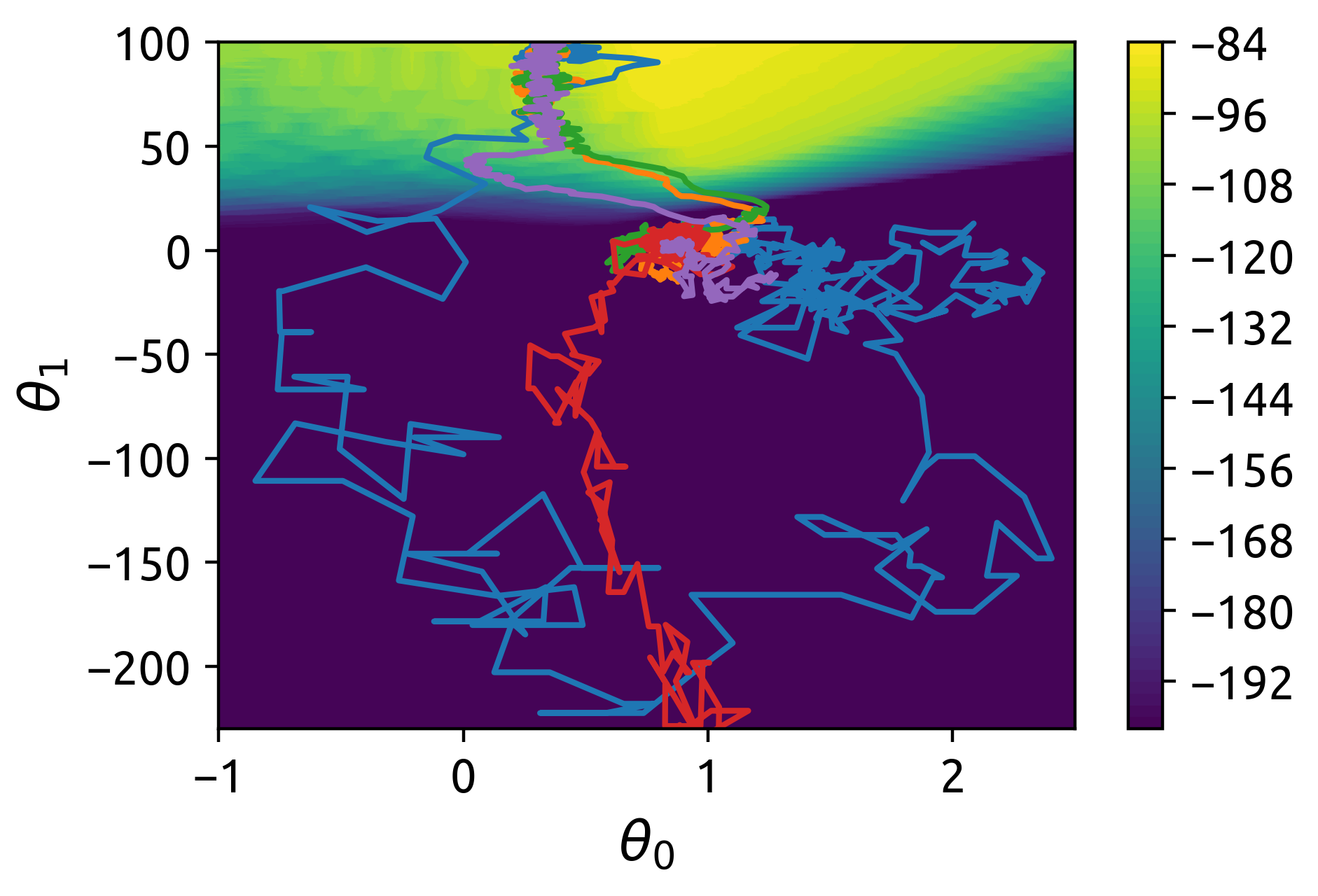}
  \vspace{-.5em}
  \caption{ }
  \label{fig:glorot_initialization}
\end{subfigure}%
\begin{subfigure}{.5\textwidth}
  \centering
  \includegraphics[width=.85\linewidth]{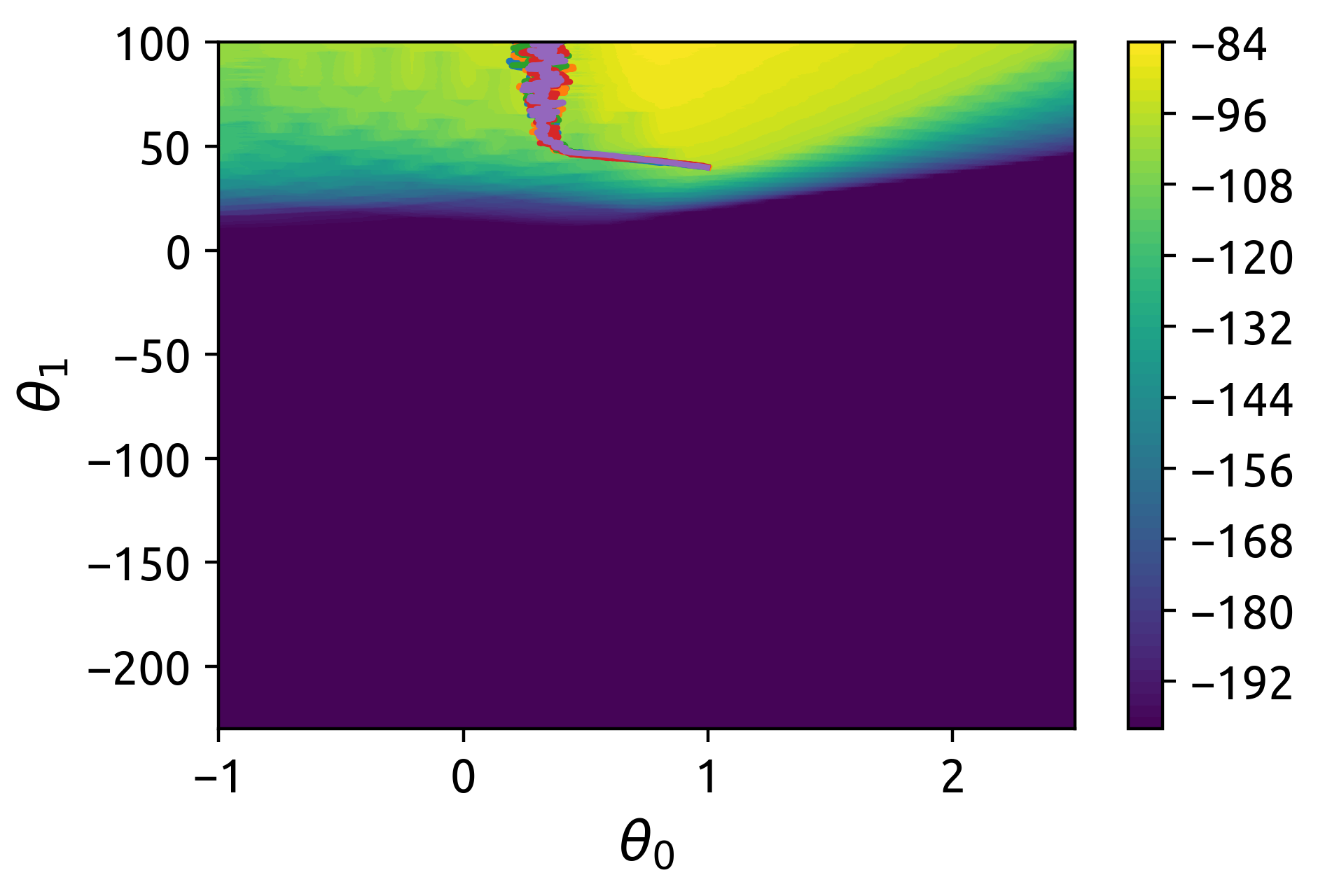}
  \vspace{-.5em}
  \caption{ }
  \label{fig:handpicked_initialization}
\end{subfigure}
\caption{Expected returns for linear policy with $2$ parameters $\theta_0, \theta_1$ on Sparse MountainCar domain (background). Gradient is $0$ in the dark blue area. Trajectories show the evolution of policy parameters over $1000$ iterations of TRPO, with $5$ random seeds. Same colors indicate the same random seeds.
(a) Random-walk type behaviour observed when parameters are initialized near the origin, following~\cite{glorot2010initialization}.
(b) Convergence observed when parameters are initialized in a region with gradients ($1,40$).
%Results suggest high variance in policy gradient methods such as TRPO stems from the exploration phase, which seems to be sensitive to policy initialization and choice of random seed.
}
\label{fig:policy_initialization}
\end{figure*}
This paper brings the following contributions:
\begin{enumerate}
    \item We relax the implicit assumptions---common in the LfD literature---that demonstrations are readily available and that the cost of acquiring them is negligible.
    \item We propose a method, called Probabilistic Planning for Demonstration Discovery (P2D2), which automatically acquires demonstrations without the need for an external expert\footnote{Code is available at \url{https://gitlab.com/tomblau/p2d2}}. The demonstrations are then used in an LfD algorithm.
    \item We provide theoretical guarantees for finding successful demonstration trajectories and derive bounds for the sampling complexity of P2D2.
    \item Experimentally, we show LfD with P2D2 outperforms both classic and intrinsic exploration techniques in RL, even when accounting for the cost of acquiring demonstrations.
    \item Results demonstrate that P2D2 lowers the variance of policy gradient methods such as TRPO, and verify that initializing policies in regions with rich gradient information makes them less sensitive to initial conditions and random seed.
\end{enumerate}

The paper is structured as follows: Section~\ref{sec:related_work} discusses related work. The P2D2 algorithm and theoretical exploration guarantees are described in Section~\ref{sec:p2d2}, followed by experimental results and analysis in Section~\ref{sec:experiments}. Finally, Section~\ref{sec:conclusion} concludes and gives directions for future work.

\section{Related work}
\label{sec:related_work}

%Learning from demonstration
Techniques for learning policies from demonstration have been extensively studied~\cite{argall2009survey}. The most basic form of LfD is pure supervised learning to imitate the actions of an expert~\cite{pomerleau1991efficient, bojarski2016end}. Policies learned in this way generalize poorly to new states not seen during training, giving rise to approaches with an RL fine-tuning phase~\cite{benbrahim1997biped}. Other methods couple LfD and RL more tightly: reinforcement learning from demonstration~\cite{brys2015reinforcement} shapes the reward function based on the demonstrated trajectories, whereas Q-learning from demonstration~\cite{hester2018deep} seeks to learn the state-action value function of the expert.

Yet another class of LfD algorithms interleaves supervision and RL. Generative adversarial reinforcement learning~\cite{ho2016generative} applies an adversarial approach, where a discriminator model learns to distinguish between expert demonstrations and policy actions, while the policy learns to fool the discriminator. The BCO algorithm~\cite{torabi2018behavioral} uses only observations of the system state, learning to infer actions and improving the policy simultaneously. All of these LfD techniques rely on user-generated demonstrations or a-priori knowledge of environment parameters. In contrast, P2D2 \emph{automatically} discovers demonstrations, with no need of an external expert. To the best of our knowledge, it is the first method with this capability.

% Classic exploration
Exploration has long been an important topic in RL. Classic techniques typically rely on adding noise to actions~\cite{mnih2015human,schulman2015trust}, and perform very poorly in settings with sparse rewards.
% Intrinsic RL
Intrinsic motivation tackles this problem by defining a new reward to direct exploration. Many intrinsic reward definitions were proposed, based on information theory \cite{oudeyer2008can}, state visitation count \cite{lopes2012exploration,bellemare2016unifying,szita2008many,fox2018dora}, value function posterior variance \cite{osband2016deep, morere2018bayesian}, or model prediction error~\cite{stadie2015incentivizing,pmlr-v70-pathak17a}.
In the setting of continuous state and action spaces none of these approaches offer guarantees for the exploration of the state space, unlike P2D2.
% BO for policy search
Recent works try to offer exploration guarantees by adapting Bayesian optimization to RL~\cite{wilson2014using,vien2018bayesian},  but results are still limited to toy problems and specific policy model classes.

Motion planning in robotics is predominantly addressed with sampling-based methods. This type of approach offers a variety of methodologies for exploration and solution space representation (e.g., \textit{Probabilistic roadmaps} (PRM)~\cite{Kavraki1996}, \textit{Expansive space trees} (EST)~\cite{Hsu1997} and \textit{Rapidly-exploring random tree} (RRT)~\cite{kuffner2000rrt}), which have shown excellent performance in path planning in high-dimensional spaces under dynamic constraints \cite{LaValle2001, Hsu2002, Kavraki1996}.

% On RL and RRT together
RL was previously combined with sampling-based planning to replace core elements of planning algorithms, such as PRM's point-to-point connection~\cite{faust2018prm}, local RRT steering function~\cite{chiang2019rl} or RRT expansion policy~\cite{chen2019learning}, and even entire motion planners~\cite{jurgensonharnessing}. In contrast, the proposed method bridges the gap in the opposite direction, employing a sampling-based planner to discover demonstrations that kick-start RL algorithms and enhance their performance.

\section{Probabilistic Planning for Demonstration Discovery (P2D2)}
\label{sec:p2d2}
This work formulates finding demonstrations as a planning problem in the state space.
Unlike random walk RL exploration, planning algorithms such as Rapidly-exploring Random Tree (RRT) encourage uniform coverage of the search space and are probabilistically complete, i.e. guaranteed to find a solution~\cite{kleinbort2019}.
The presented algorithm, called P2D2, generates a collection of successful trajectories, which can be used instead of expert demonstrations in LfD algorithms.
We first provide preliminaries in Section~\ref{sec:preliminaries}, then present the RRT algorithm in Section~\ref{sec:rrt} and adapt it to the RL setting in order to achieve automatic demonstration discovery. Finally, theoretical guarantees for the proposed method are provided in Section~\ref{sec:exploration_guarantees}.

\subsection{Preliminaries}
\label{sec:preliminaries}
This work is based on the Markov Decision Process (MDP) framework, defined as a tuple ${<\mathcal{S},\mathcal{A},T,R,\gamma>}$. $\mathcal{S}$ and $\mathcal{A}$ are spaces of states $s$ and actions $a$ respectively. $T: \mathcal{S} \times \mathcal{A} \times \mathcal{S} \rightarrow [0,1]$ is a transition probability distribution so that $T(s_t,a_t,s_{t+1})=p(s_{t+1}|s_t,a_t)$, where subscript $t$ indicates the $t^{th}$ discrete timestep.
$R:\mathcal{S} \times \mathcal{A} \times \mathcal{S} \rightarrow \mathbb{R}$ is a reward function defining rewards $r_t$ associated with transitions $(s_t, a_t, s_{t+1})$. $\gamma \in [0,1)$ is a discount factor, and we denote the space of initial states $\mathcal{S}_0$. Solving a MDP amounts to finding the optimal policy $\pi^*$ maximizing the expected return $J(\pi^*) = \mathbb{E}_{T,\pi^*}[\sum^\infty_{t=0}\gamma^tR(s_t,a_t,s_{t+1})]$, where actions are chosen according to $a_t = \pi^*(s_t)$.

We add the following definitions to the MDP framework:

\begin{mydef}\label{def:transition}
A transition $(s_t, a_t, s_{t+1})$ is \emph{valid} if and only if $T(s_t,a_t,s_{t+1}) > 0$. That is, the transition is in the support of transition dynamics $T$.
\end{mydef}

\begin{mydef}\label{def:traj}
A valid trajectory is a sequence $\tau = \left[s_0, a_0, s_1,\ldots,s_{t_\tau-1}, a_{t_\tau-1}, s_{t_\tau}\right]$ such that $(s_t, a_t, s_{t+1})$ is a valid transition $\forall t \in \{0,1,\ldots,t_\tau-1\}$, and $s_0 \in \mathcal{S}_0$.
\end{mydef}
Additionally, whenever a goal space $\mathcal{S}_{goal} \subseteq  \mathcal{S}$ is defined for the MDP, we say a trajectory is \emph{successful} if its end state is within the goal space, i.e. $s_{t_\tau} \in \mathcal{S}_{goal}$.

\subsection{Planning in MDPs with RRT}
\label{sec:rrt}
The RRT algorithm~\cite{LaValle2001} provides a principled approach for planning in problems that cannot be solved directly (e.g. using inverse kinematics), but where it is possible to sample transitions.
% complex motion when the inverse kinematics are unknown\footnote{Commonly referred to as a steering function, which returns the required action to move an agent between neighbouring states.}.
Ordinarily, a planning problem given to an RRT is defined by a configuration space $\mathcal{C}$, a set of configurations that result in collision $\mathcal{C}_{obs}$, and the set of collision-free configurations $\mathcal{C}_{free} = \mathcal{C} \setminus \mathcal{C}_{obs}$. RRT and variant algorithms commonly assume that it is possible to travel in a straight line in $\mathcal{C}_{free}$, or that there is access to a \textit{steering function} that can move an agent between the end-points of such a line segment. In each iteration of the algorithm, RRT randomly samples a candidate point  $c_{rand} \in \mathcal{C}_{free}$ and tries to connect it to the nearest node in the tree $c_{near} \in \mathbb{T}$ via a straight line (typically what ends up being added to the tree is a point $c_{new}$ on the line $\overline{c_{near}c_{rand}}$).

In the MDP setting, the state space $\mathcal{S}$ replaces the configuration space $\mathcal{C}$, and obstacles are only defined implicitly by the transition dynamics $T$. In general it is not possible to travel along a straight line, and a steering function is usually not available. For these reasons, an RRT planning in MDPs can't connect nodes in a straight line. Instead, after sampling candidate state $s_{rand} \in \mathcal{S}$, the RRT will add a new node $s_{new}$ connected to $s_{near}$ by executing a random action $a_{rand} \in \mathcal{A}$ at the state $s_{near}$. The edge $(s_{near},s_{new})$ will store this action. If the newly added state satisfies $s_{new} \in \mathcal{S}_{goal}$, the algorithm is finished, and a successful trajectory can be generated by traversing up the tree from $s_{new}$ to $s_{root}$. A typical instance of this procedure is illustrated in Figure~\ref{fig:rrt_trajectory}.

The ability to execute this modified RRT algorithm in MDPs relies on two assumptions:
\begin{assumption}\label{as:1}
States can be sampled uniformly from the MDP state space $\mathcal{S}$.
\end{assumption}

Since $\mathcal{S}$ is typically a hyper-rectangle, sampling from it uniformly is trivial. Note that the sampled states don't have to be valid (i.e. collision-free). Even if $s_{rand}$ isn't valid, $s_{new}$ and $(s_{near},s_{new})$ will be valid because they are produced by the transition dynamics.

\begin{assumption}\label{as:2}
The environment state can be set to a previously visited state $s \in \mathbb{T}$ (e.g. by means of a simulator).
\end{assumption}
Although this assumption may seem limiting at first glance, it is already in use in the RL literature~\cite{florensa2017reverse,nair2018overcoming,ecoffet2019go}. Further, it lines up nicely with existing research on sim-to-real transfer~\cite{peng2018sim}: a policy can be trained in simulation, where the assumption is easily satisfied, then transferred to physical environments using sim-to-real techniques.

The proposed P2D2 algorithm, as summarized in Algorithm~\ref{alg:p2d2}, discovers a set of $N$ demonstrations ${\bm{\tau}=\{\tau_i\}_1^N}$ by repeatedly executing the above modified RRT with different initial states. LfD techniques can then be leveraged to learn an imitation policy $\pi_0$, which may be further refined using traditional RL algorithms such as TRPO~\cite{schulman2015trust}.

\begin{algorithm}[!b]
	\caption{\small Probabilistic Planning for Demonstration Discovery}
	\label{alg:p2d2}
	\begin{algorithmic}
        \STATE {\bfseries Input:} $\mathcal{S}_{goal}, N, \mathcal{M}$: MDP, $k$: sampling budget
        \STATE {\bfseries Input(optional):} $p_g$: goal sampling prob.
        \STATE {\bfseries Output:} ${\bm{\tau}=\{\tau_i\}_1^N}$: successful trajectories
        \STATE $\bm{\tau} \leftarrow \emptyset$
        \WHILE{$\abs{\bm{\tau}} < N$}
            \STATE $s_0 \leftarrow \text{sample initial state } s_0 \sim \mathbb{U}(\mathcal{S}_0)$        
            \STATE Initialise tree $\mathbb{T}$ with root $s_0$
            \FOR{$i = 1:k$}
                \STATE $s_{rand} \leftarrow \text{sample random state }s_{rand} \sim \mathbb{U}(\mathcal{S})$
                \IF{$u \sim \mathbb{U}(\left[0,1\right]) \leq p_g$}
                    \STATE $s_{rand} \leftarrow \text{sample } s_{rand}\sim \mathbb{U}(\mathcal{S}_{goal})$
                \ENDIF
        	    \STATE $s_{near} \leftarrow \text{find nearest node to $s_{rand}$ in $\mathbb{T}$}$
        	    \STATE $a \leftarrow \text{sample random action } a \sim \mathbb{U}(\mathcal{A})$
        	    \STATE $s_{new} \leftarrow \text{execute $a$ in state $s_{near}$}$
        	    \STATE Add $a$, node $s_{new}$ and edge ($s_{near},s_{new}$) to $\mathbb{T}$
            \ENDFOR
            \STATE $\tau \leftarrow \text{trajectory in $\mathbb{T}$ with max. cumulated reward}$
            \IF{$\tau$ is a successful trajectory}
                \STATE $\bm{\tau} \leftarrow \bm{\tau} \cup \tau$
            \ENDIF
        \ENDWHILE

    \end{algorithmic}
% 	\DontPrintSemicolon
% 	\KwIn{$s_0$, $k$: sampling budget}
% 	\myinput{(optional) $\mathcal{S}_{goal}$, $p_g$: goal sampling prob.}
% 	\KwOut{$\tau$: successful trajectory}
    
%     Add root node $s_0$ to $\mathbb{T}$\\
%     \For{$i = 1:k$}{
% 	    $s_{rand} \leftarrow \text{sample random state }s_{rand} \in \mathcal{S}$\\
% 	    \IF{$u \sim \mathbb{U}(0,1) \leq p_g$}{
% 	    $s_{rand} \leftarrow \text{sample $s_{rand}$ from $\mathcal{S}_{goal}$}$\\}
% 	    $s_{near} \leftarrow \text{find nearest node to $s_{rand}$ in $\mathbb{T}$}$\\
	    
% 	    $a \leftarrow \text{sample } \pi_l(s_{near}, s_{rand}-s_{near})$\\
% 	    $s_{new} \leftarrow \text{execute $a$ in state $s_{near}$}$\\
% 	    update $\pi_l$ with $(\{s_{near},s_{new}-s_{near}\}, a)$\\
	    
% 	    %\If{$s_{new}$\text{ is not a failure state}}{
% 	    Add node $s_{new}, a$ and edge ($s_{near},s_{new}$) to $\mathbb{T}$\\
% 	    %}
% 	}
% 	$\tau \leftarrow \text{trajectory in $\mathbb{T}$ with max. cumulated reward}$\\
\end{algorithm}
\begin{figure}[t]
  \centering
  \includegraphics[width=.95\linewidth]{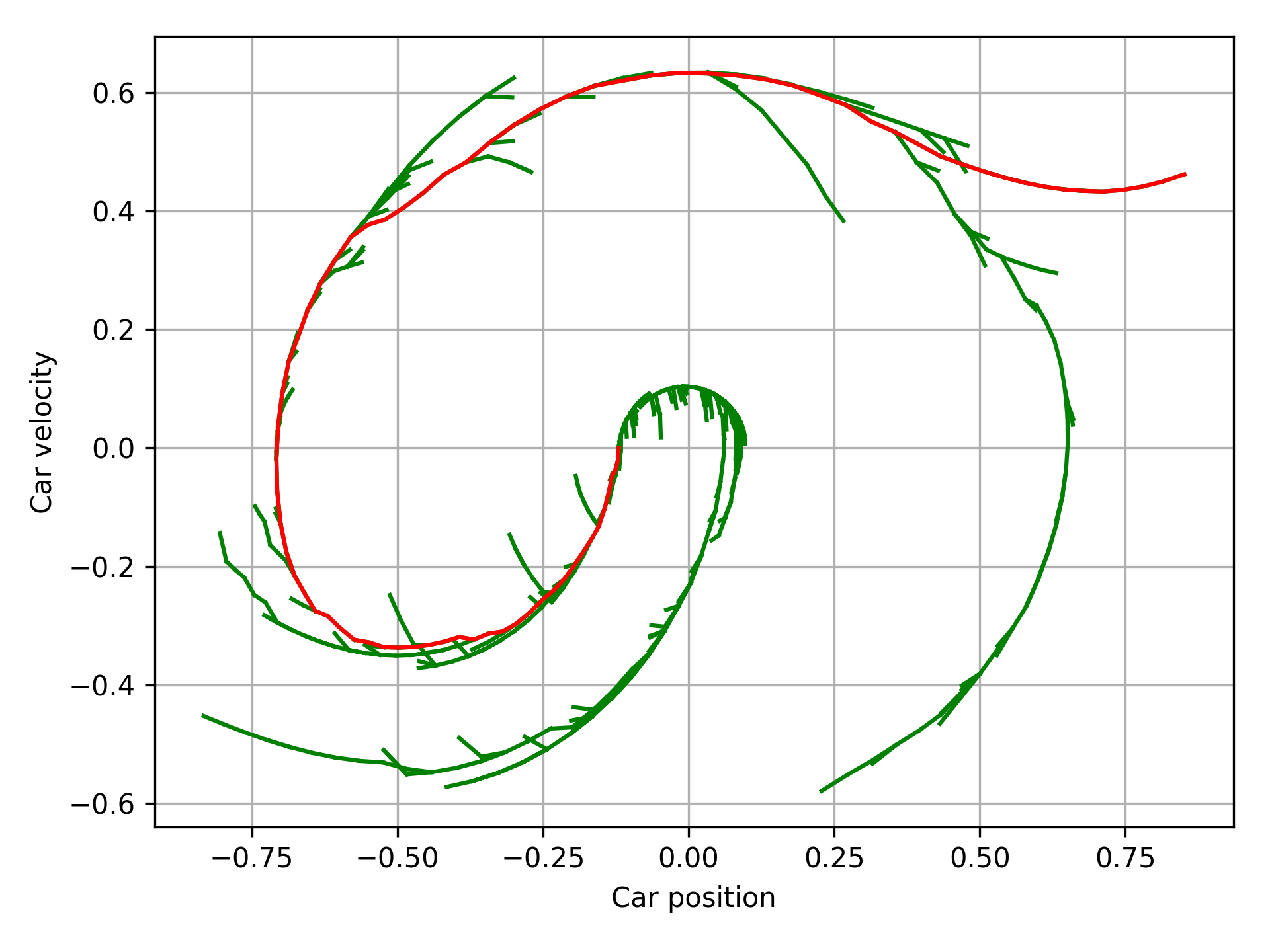}
  \captionof{figure}{Example of P2D2 on sparse MountainCar. Green segments are sampled transitions, executed in simulation. A successful solution found by P2D2 is displayed in red. State dimensions are normalized to $\left[-1,1\right]$.} \label{fig:rrt_trajectory}
\end{figure}

\subsection{Exploration guarantees}
\label{sec:exploration_guarantees}
The RL planning environment defines differential constraints of the form:
\begin{equation}\label{eq:diff_cont}
\dot{s} = f(s(t), a(t)), \quad s(t) \in \mathcal{S}, \quad a(t) \in \mathcal{A}.
\end{equation}
Therefore, starting at $s_0$, the trajectory $\tau$ can be computed by forward integrating equation~(\ref{eq:diff_cont}) with the applied actions.
As with many RL problems, $a(t)$ is time-discretized resulting in a piecewise constant control function. This means $\tau$ is constructed of $n_{\tau}$ segments of fixed time duration $\Delta t$ such that the overall trajectory duration $t_{\tau} = n_{\tau}\cdot \Delta t$. Thus, $a(t)$ is defined as $a(t)=a_i \in \mathcal{A}$ where $t \in [(i-1) \cdot \Delta t, i \cdot \Delta t)$ and $ 1\leq i \leq n_{\tau}$. Furthermore, as all transitions between states in $\tau$ are known, the trajectory return can be defined as $R_{\tau}=\sum^{n_{\tau}}_{t=0}\gamma^tR(s_t,a_t,s_{t+1})$. %\textcolor{red}{$H$ and $n_{\tau}$ are the same?}

P2D2 explores in state-action space instead of policy parameter space. Furthermore, it is an effective exploration framework which provides \textit{probabilistic completeness} (PC):
\begin{mydef}\label{def:pc}
A probabilistically complete planner finds a feasible solution (if one exists) with a probability approaching 1 in  the limit of infinite samples. 
\end{mydef}

With the aforementioned dynamic characteristics, we prove that P2D2 under the RL setting is PC. This is in stark contrast to the random walk-like RL exploration, discussed in section \ref{sec:intro}, which is \emph{not} PC.
We begin with the following theorem, a modification of Theorem 2 from~\cite{kleinbort2019}, which is applied to kinodynamic RRT where a goal set $\mathcal{S}_{goal}$ is defined.
\begin{theorem}\label{lemma:rrt}
	Suppose that there exists a valid trajectory $\tau$ from $s_0$ to $\mathcal{S}_{goal}$ as defined in definition~\ref{def:traj}, with a corresponding piecewise constant control. The probability that P2D2 fails to reach $\mathcal{S}_{goal}$ from $s_0$ after $k$ iterations is bounded by $a e^{-bk}$, for some constants $a,b > 0$.
\end{theorem} 
The proof, which is a modification of Theorem 2 from \cite{kleinbort2019}, can be found in Appendix~\ref{app:B}. It should be noted that P2D2 does not require an explicit definition for $\mathcal{S}_{goal}$ in order to explore the space. While in some path planning variants of RRT, $\mathcal{S}_{goal}$ is used to bias sampling, the main purpose of $\mathcal{S}_{goal}$ is to indicate that a solution has been found. Therefore, $\mathcal{S}_{goal}$ can be replaced by another implicit success criterion. In the RL setting, this can be replaced by a return-related criterion.
\begin{theorem}\label{theorem:rrt_rl}
	Suppose there exists a trajectory with a return $R_{\tau} \geq \hat{R}, \hat{R} \in \mathbb{R}$. The probability that P2D2 fails to find a valid trajectory from $s_0$ with $R_{\tau} \geq \hat{R}$ after $k$ iterations is bounded by $\hat{a} e^{-\hat{b} k}$, for some constants $\hat{a},\hat{b} > 0$.
\end{theorem}
\begin{proof}
The proof is straightforward. We augment each state in $\tau$ with the return for reaching it from $s_0$:
\begin{align}
s'_n &= \begin{bmatrix}
s_n  \\
R_{s_n}
\end{bmatrix}, \qquad \forall n=1:n_{\tau},
\end{align}
where $R_{s_n}=\sum^{n}_{t=0}\gamma^tR(s_t,a_t,s_{t+1})$. For consistency we modify the distance metric by simply adding a reward distance metric. With the above change in notation, we modify the goal set to $\mathcal{S}^{RL}_{goal}=\{(s,R_s) | s\in \mathcal{S}_{goal}, R_s \geq \hat{R} \}$, such that there is an explicit criterion for minimal return as a goal. Consequently, the exploration problem can be written for the augmented representation as $(\mathcal{S},s^{RL}_{0}, \mathcal{S}^{RL}_{goal})$, where $s^{RL}_{0} = [s_0, 0]^\top$. Theorem~\ref{lemma:rrt} satisfies that P2D2 can find a feasible solution to this problem within finite time, i.e. PC, and therefore the probability of not reaching $\mathcal{S}^{RL}_{goal}$ after $k$ iterations is upper-bounded by the exponential term $\hat{a} e^{-\hat{b} k}$, for some constants $\hat{a},\hat{b} > 0$
\end{proof}

We can now state our main result on the sampling complexity of the exploration process.
\begin{theorem}\label{theorem:rrt_complex}
	If trajectory exploration is probabilistically complete and satisfies an exponential convergence bound, the expected sampling complexity is finite and bounded s.t. 
	\begin{equation}
    \mathbb{E}[k] \leq \frac{\hat{a}}{4\sinh^{2}{\frac{\hat{b}}{2}}},
\end{equation}
where $\hat{a},\hat{b} > 0$.	
\end{theorem}
\begin{proof}
Theorem~\ref{theorem:rrt_rl} provides an exponential bound for the probability the planner fails in finding a feasible path. Hence, we can bound the expected number of iterations needed to find a solution, i.e. sampling complexity:
\begin{align}
    \mathbb{E}[k] \leq \sum^{\infty}_{k=1} k\hat{a} e^{-\hat{b} k} &= 
    \sum^{\infty}_{k=1} -\hat{a}\frac{d e^{-\hat{b} k}}{d\hat{b}}\\
    &= -\hat{a}\frac{d}{d\hat{b}}\sum^{\infty}_{k=1} e^{-\hat{b} k}\\
    &= -\hat{a}\frac{d}{d\hat{b}} \frac{1}{e^{\hat{b}}-1}\\
    &= \frac{\hat{a}}{4\sinh^{2}{\frac{\hat{b}}{2}}},
\end{align}
where we used the relation $\sum^{\infty}_{k=1} e^{-\hat{b} k} =\frac{1}{e^{\hat{b}}-1}$.
\end{proof}
It is worth noting that while the sample complexity is bounded, the above result implies that the bound varies according to problem-specific properties, which are encapsulated in the value of $\hat{a}$ and $\hat{b}$. Intuitively, $\hat{a}$ depends on the scale of the problem. It grows as $\abs{\mathcal{S}^{RL}_{goal}}$ becomes smaller or as the length of the solution trajectory becomes longer. $\hat{b}$ depends on the probability of sampling states that will expand the tree in the right direction. It therefore shrinks as the dimensionality of $\mathcal{S}$ increases. We refer the reader to Appendix~\ref{app:B} for more details on the meaning of $\hat{a},\hat{b}$ and the derivation of the tail bound in Theorem~\ref{lemma:rrt}.

% The success of sampling-based planners is mainly due to their probabilistic completeness, which guarantees that planners return a solution, if one exists, in a finite time (see also definition \ref{def:pc}). A more desirable property, albeit not available in all planners, is asymptotic optimality, which states that a planner will return an optimal solution as the number of samples tends to infinity \cite{Karaman2011, karaman2010incremental}. However, in the RL setting, where the steering function is unknown and forward dynamics parameters (e.g. execution time) are fixed, the asymptotic optimality guarantees of classic sampling-based planners do not apply.

\begin{figure*}[b]
\vspace{-0.5em}
\centering
\begin{subfigure}{.33\textwidth}
  \centering
  \includegraphics[width=\linewidth]{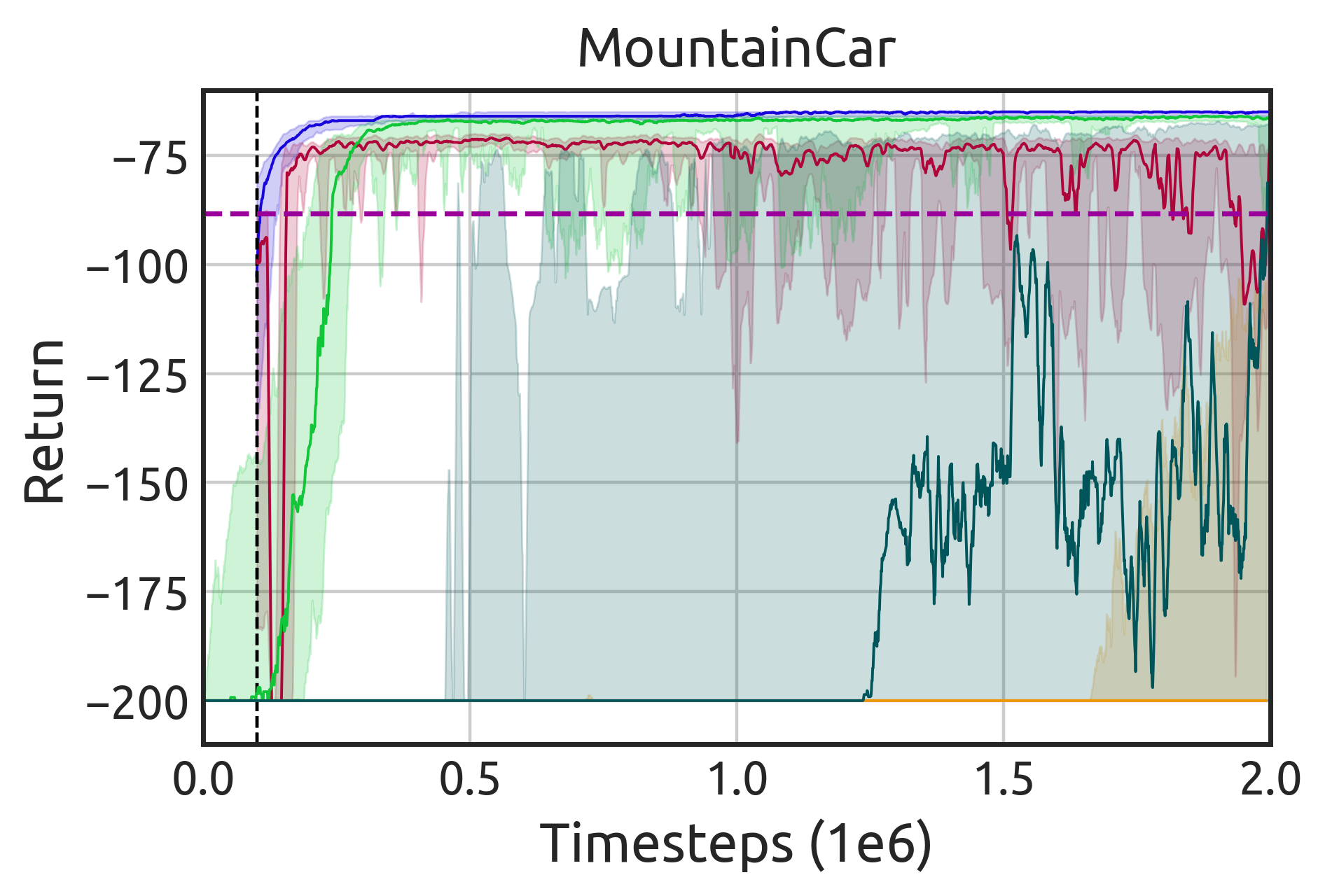}
\end{subfigure}%
\begin{subfigure}{.33\textwidth}
  \centering
  \includegraphics[width=\linewidth]{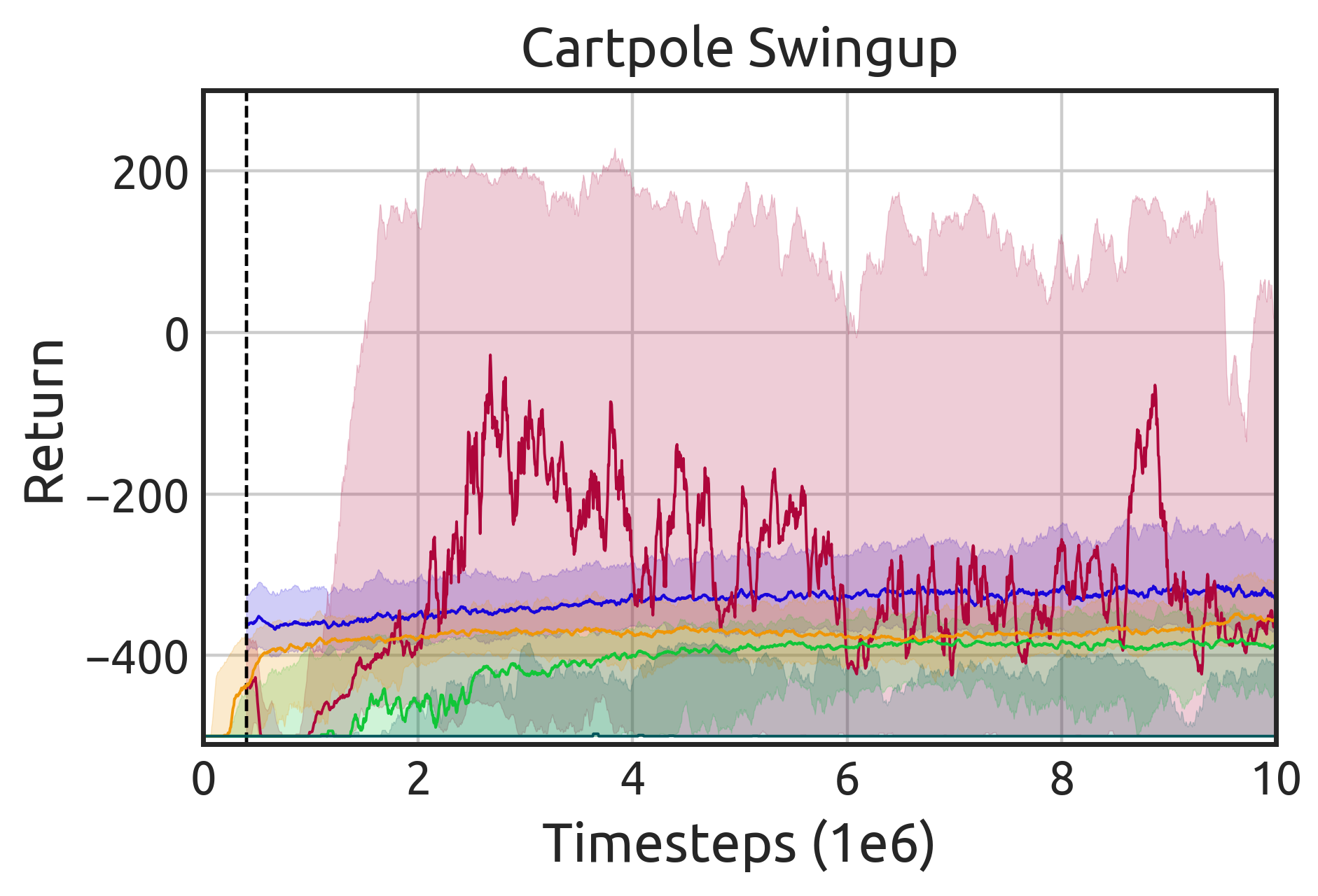}
\end{subfigure}%
\begin{subfigure}{.33\textwidth}
  \centering
  \includegraphics[width=\linewidth]{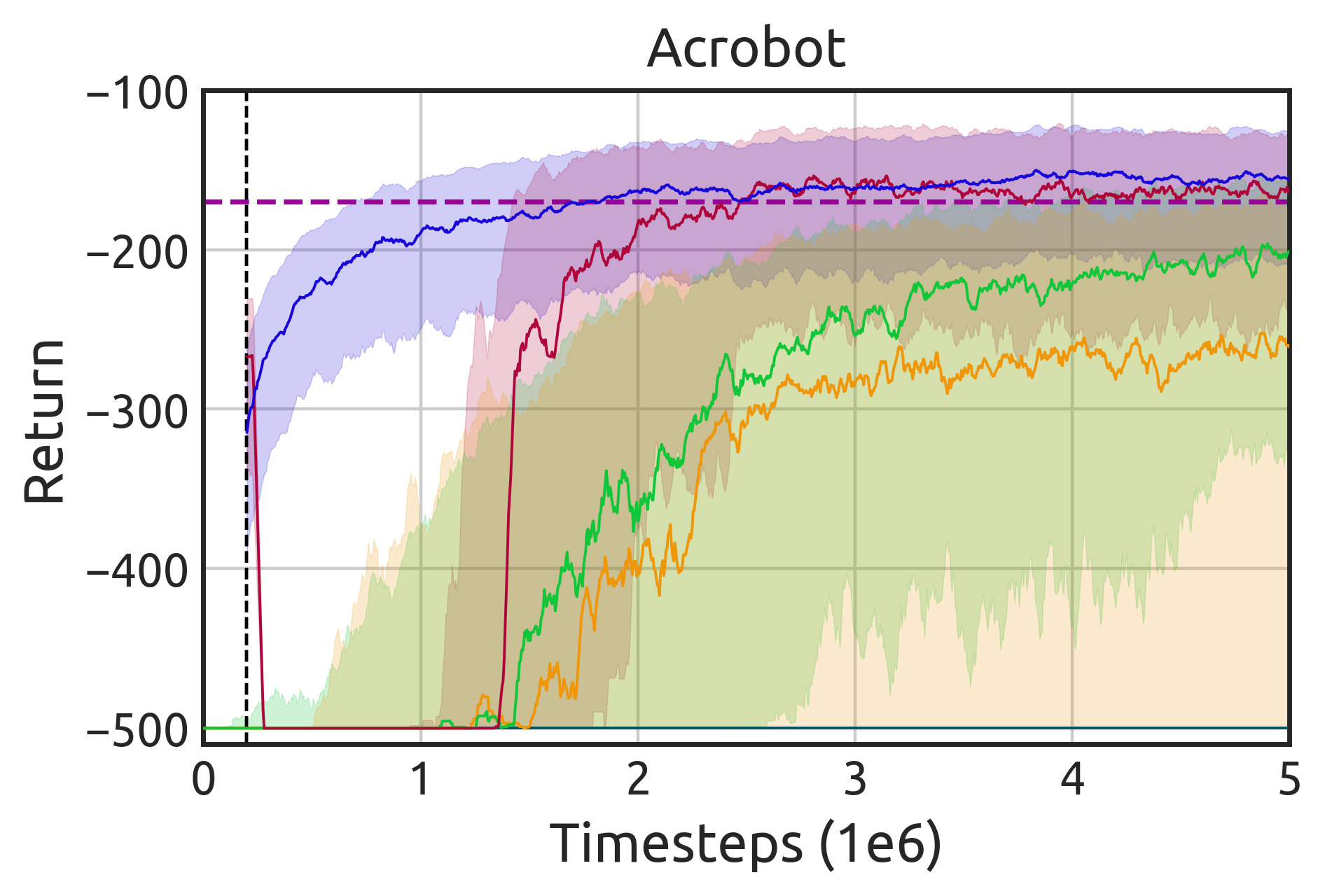}
\end{subfigure}
\begin{subfigure}{.33\textwidth}
  \centering
  \includegraphics[width=\linewidth]{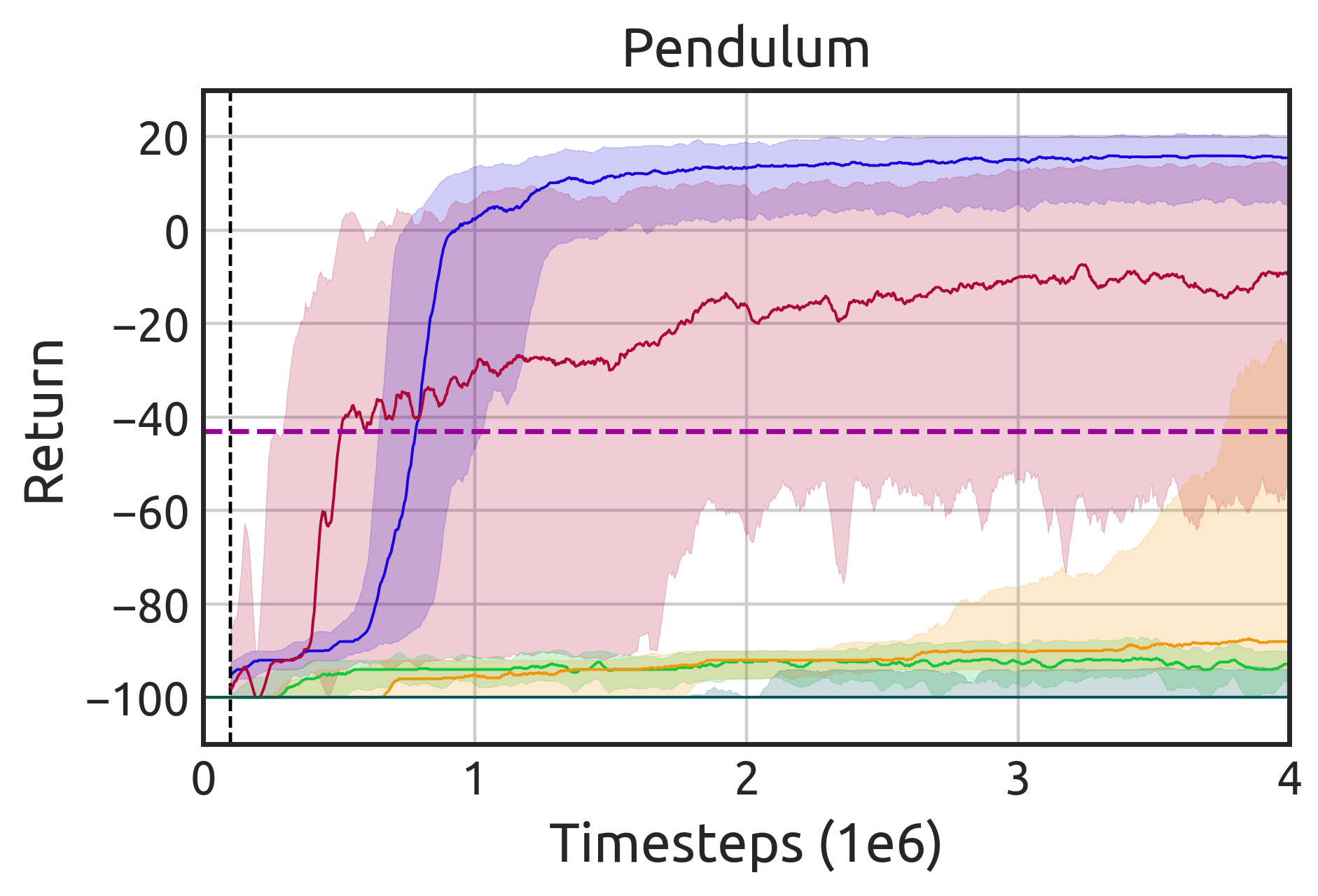}
\end{subfigure}%
\begin{subfigure}{.33\textwidth}
  \centering
  \includegraphics[width=\linewidth]{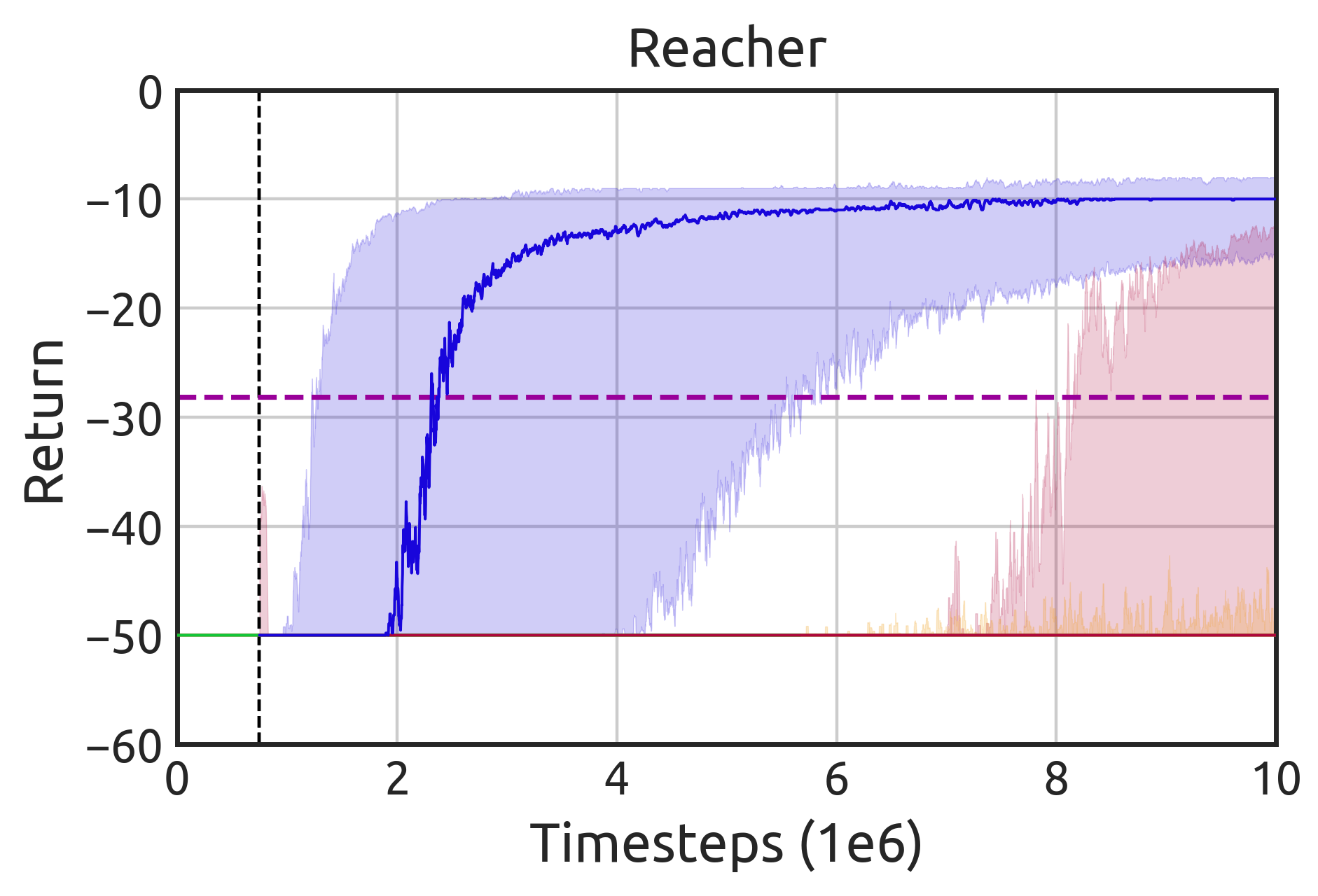}
\end{subfigure}%
\begin{subfigure}{.33\textwidth}
  \centering
  \includegraphics[width=\linewidth]{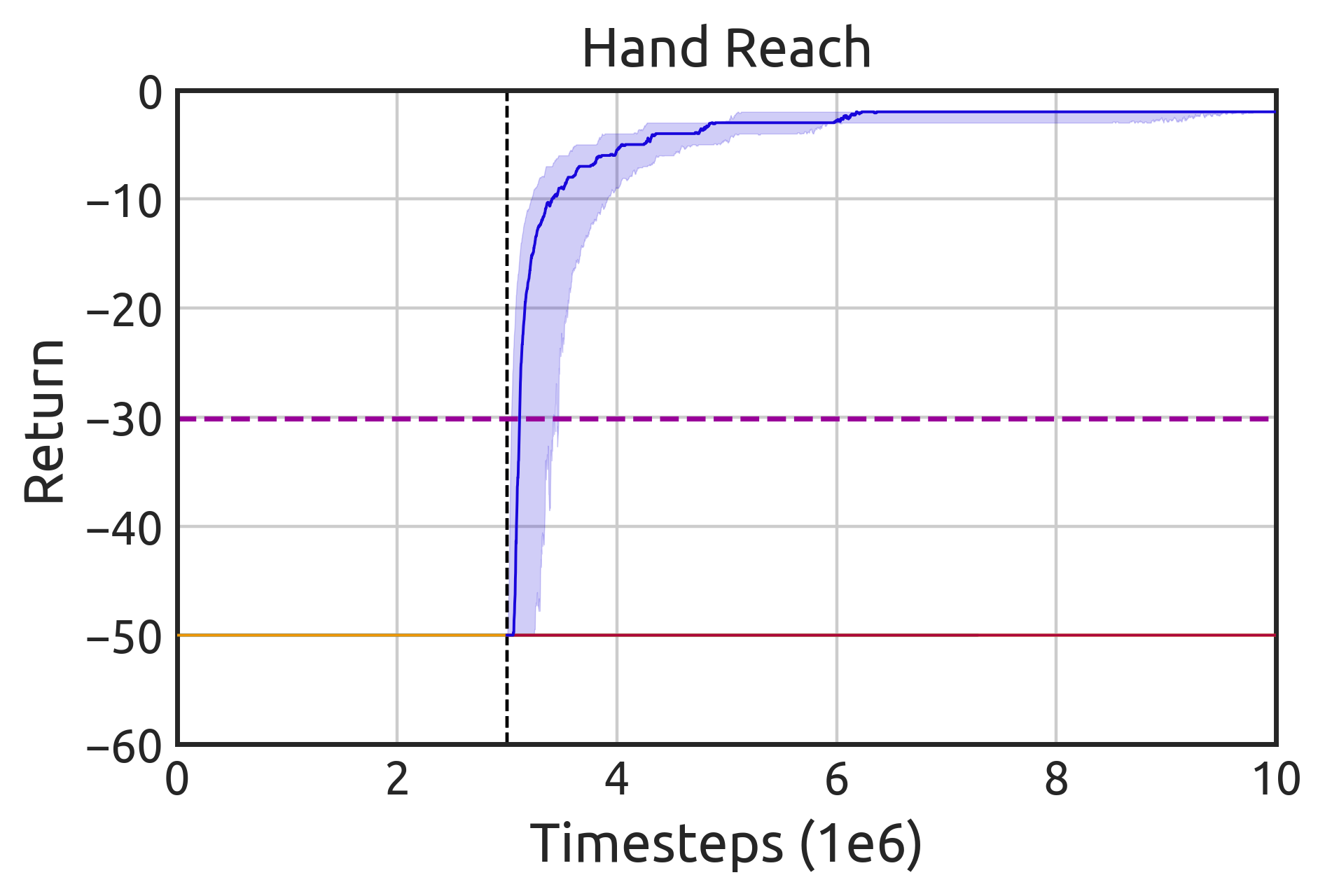}
\end{subfigure}
\begin{subfigure}{.33\textwidth}
  \centering
  \includegraphics[width=\linewidth]{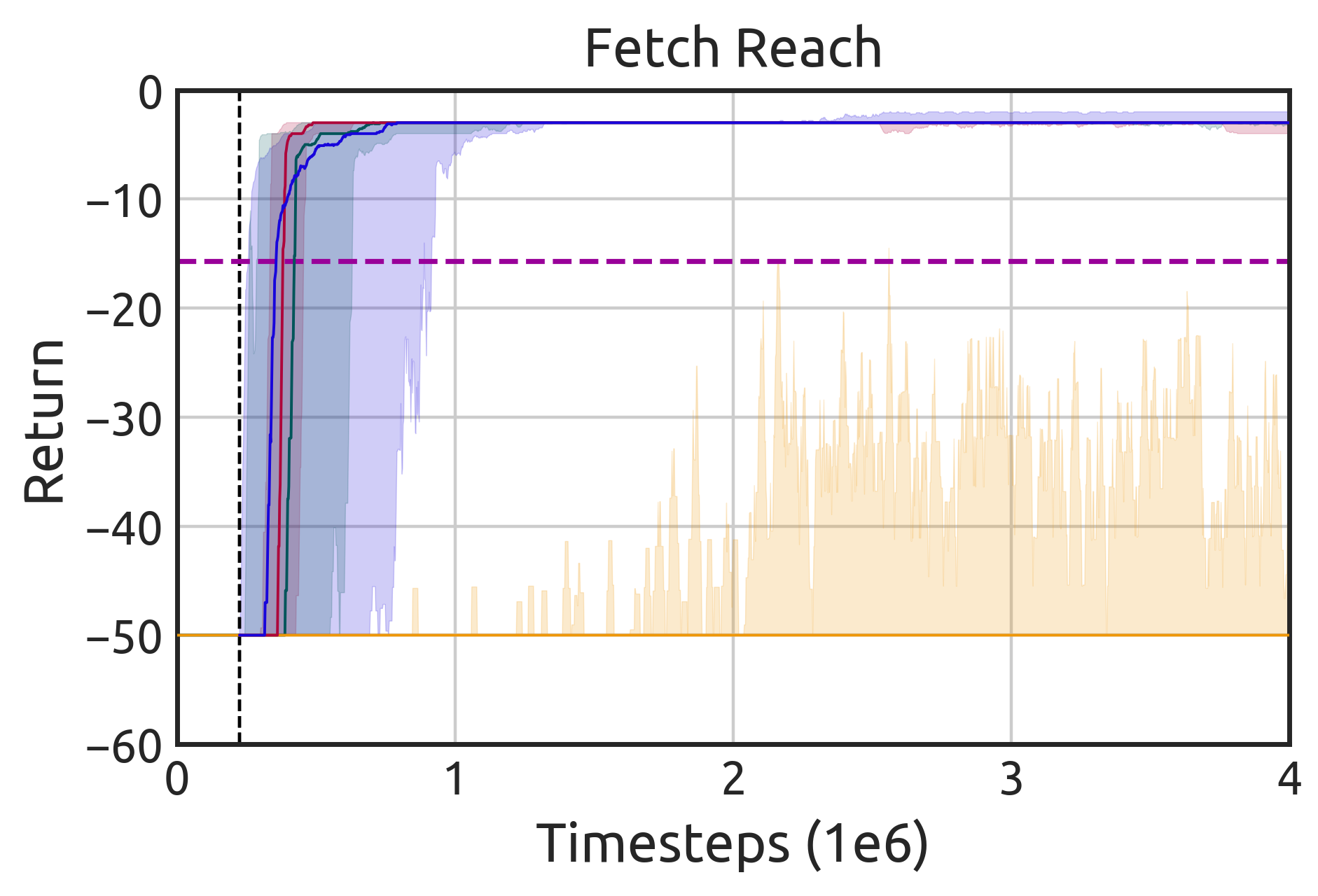}
\end{subfigure}%
\begin{subfigure}{.33\textwidth}
  \centering
  \includegraphics[width=\linewidth]{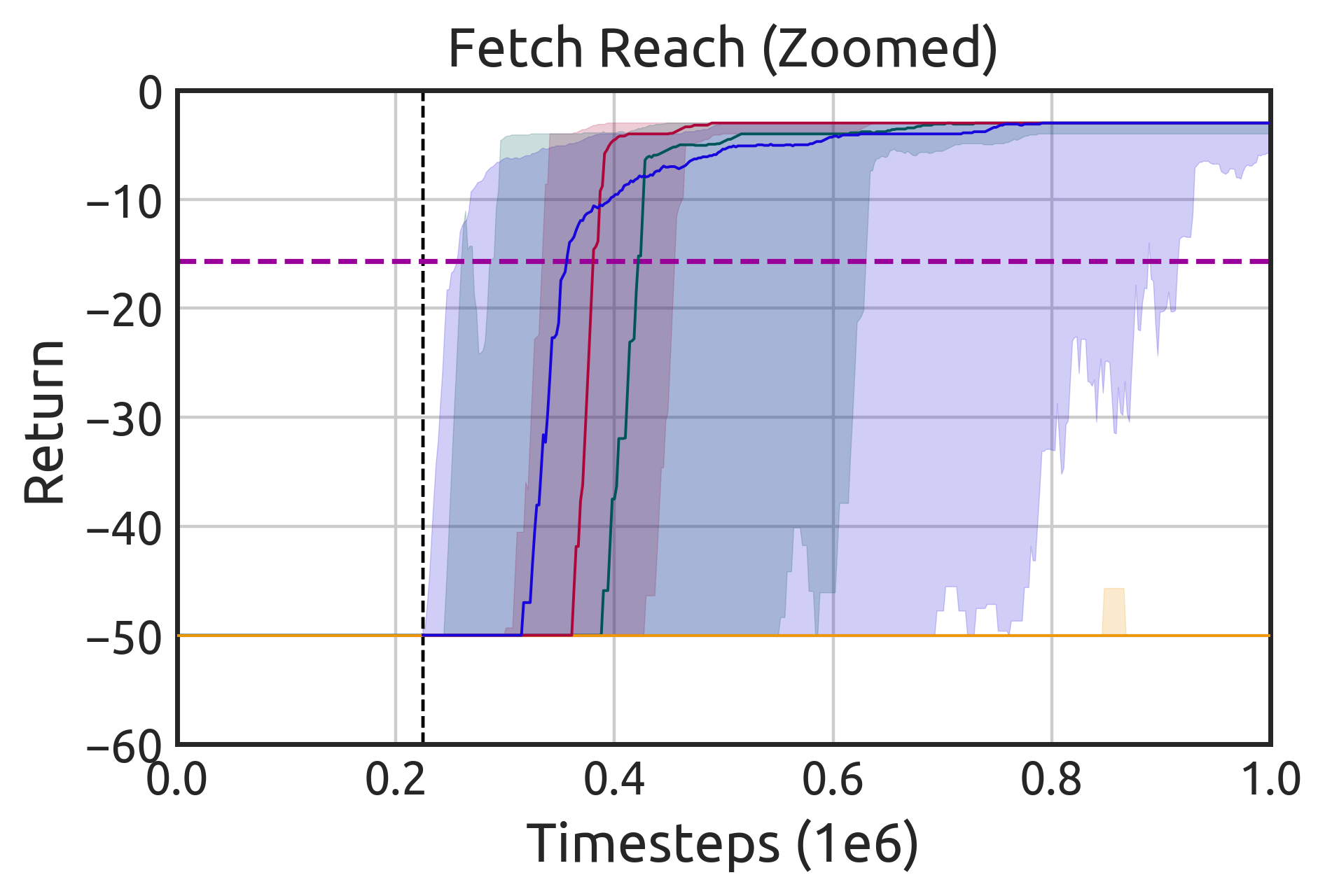}
\end{subfigure}%
\begin{subfigure}{.33\textwidth}
  \centering
  \includegraphics[width=\linewidth]{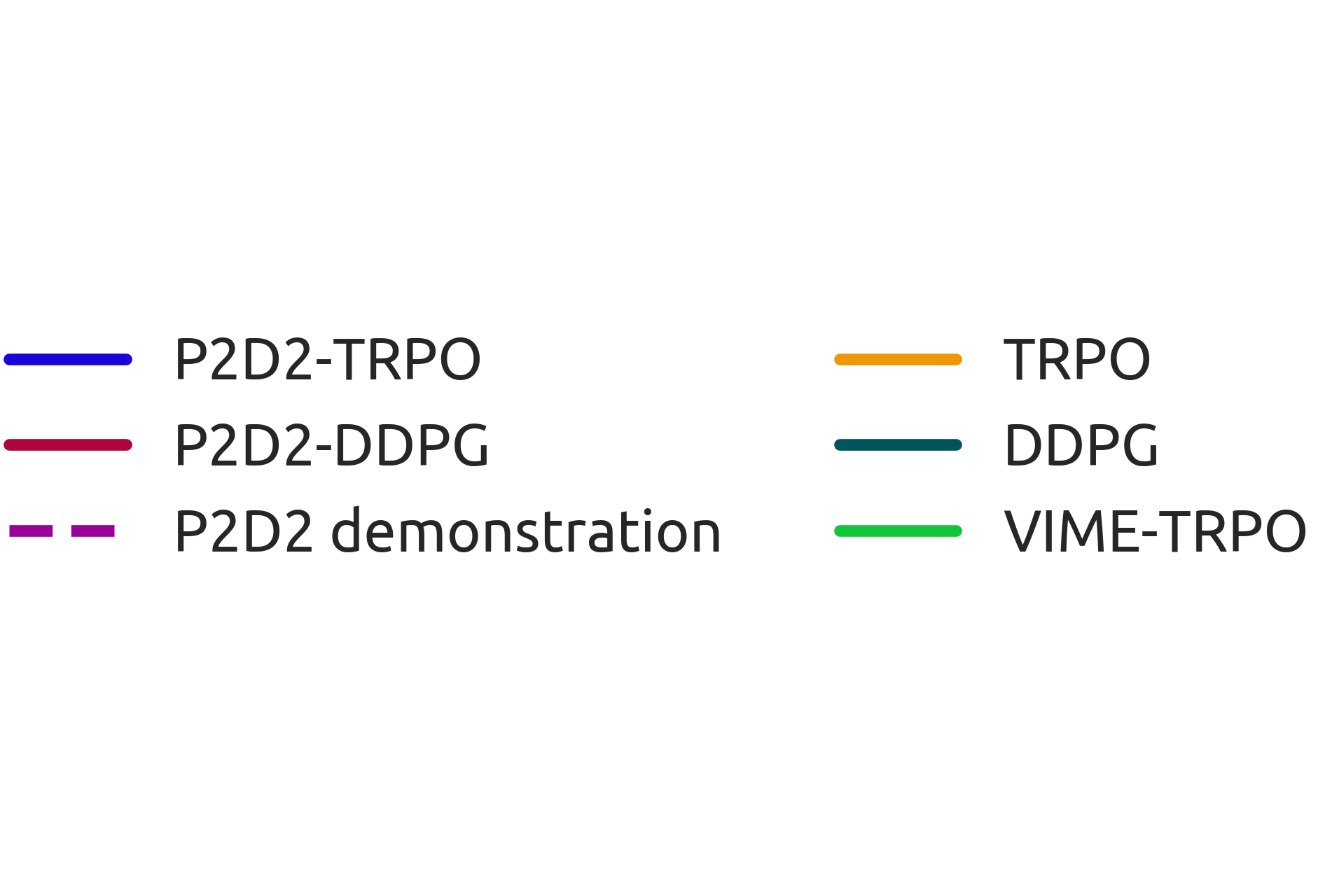}
\end{subfigure}
\caption{Results for classic control tasks, comparing our proposed method (P2D2-TRPO/DDPG), vanilla TRPO/DDPG, and VIME-TRPO. Trendlines are the medians and shaded areas are the interquartile range, taken over $10$ randomly chosen seeds. Also shown is the average undiscounted return of successful trajectories discovered by P2D2. The dashed offset at the start of P2D2-TRPO/DDPG reflects the number of timesteps spent on P2D2.}
\label{fig:results}
\end{figure*}

\section{Experiments}
\label{sec:experiments}

%MAYDO: fix the white space around figure

% Experiment goal
In this section, we investigate (i) whether P2D2 reduces the number of exploration samples needed to find good policies, compared with methods using classic and intrinsic exploration, (ii) how P2D2 can reduce the variance associated with policy gradient methods, and (iii) how P2D2 compares with standard LfD.
% Env description
All experiments make use of the Garage~\cite{garage} and Gym~\cite{gym} frameworks. The experiments feature sparse-reward versions of the following tasks:
\emph{MountainCar} ($\mathcal{S} \subseteq \mathbb{R}^2, \mathcal{A} \subseteq \mathbb{R}$),
\emph{Pendulum} ($\mathcal{S} \subseteq \mathbb{R}^2, \mathcal{A} \subseteq \mathbb{R}$),
\emph{Cartpole Swingup} ($\mathcal{S} \subseteq \mathbb{R}^4, \mathcal{A} \subseteq \mathbb{R}$),
\emph{Acrobot} ($\mathcal{S} \subseteq \mathbb{R}^4, \mathcal{A} \subseteq \mathbb{R}$).
Experiments also include the following robotics tasks:
\emph{Reacher} ($\mathcal{S} \subseteq \mathbb{R}^6, \mathcal{A} \subseteq \mathbb{R}^2$)
\emph{Fetch Reach} ($\mathcal{S} \subseteq \mathbb{R}^{13}, \mathcal{A} \subseteq \mathbb{R}^4$),
and \emph{Hand Reach} ($\mathcal{S} \subseteq \mathbb{R}^{78}, \mathcal{A} \subseteq \mathbb{R}^{20}$).

\begin{figure*}[!t]
\centering
\begin{subfigure}{.33\textwidth}
  \centering
  \includegraphics[width=\linewidth]{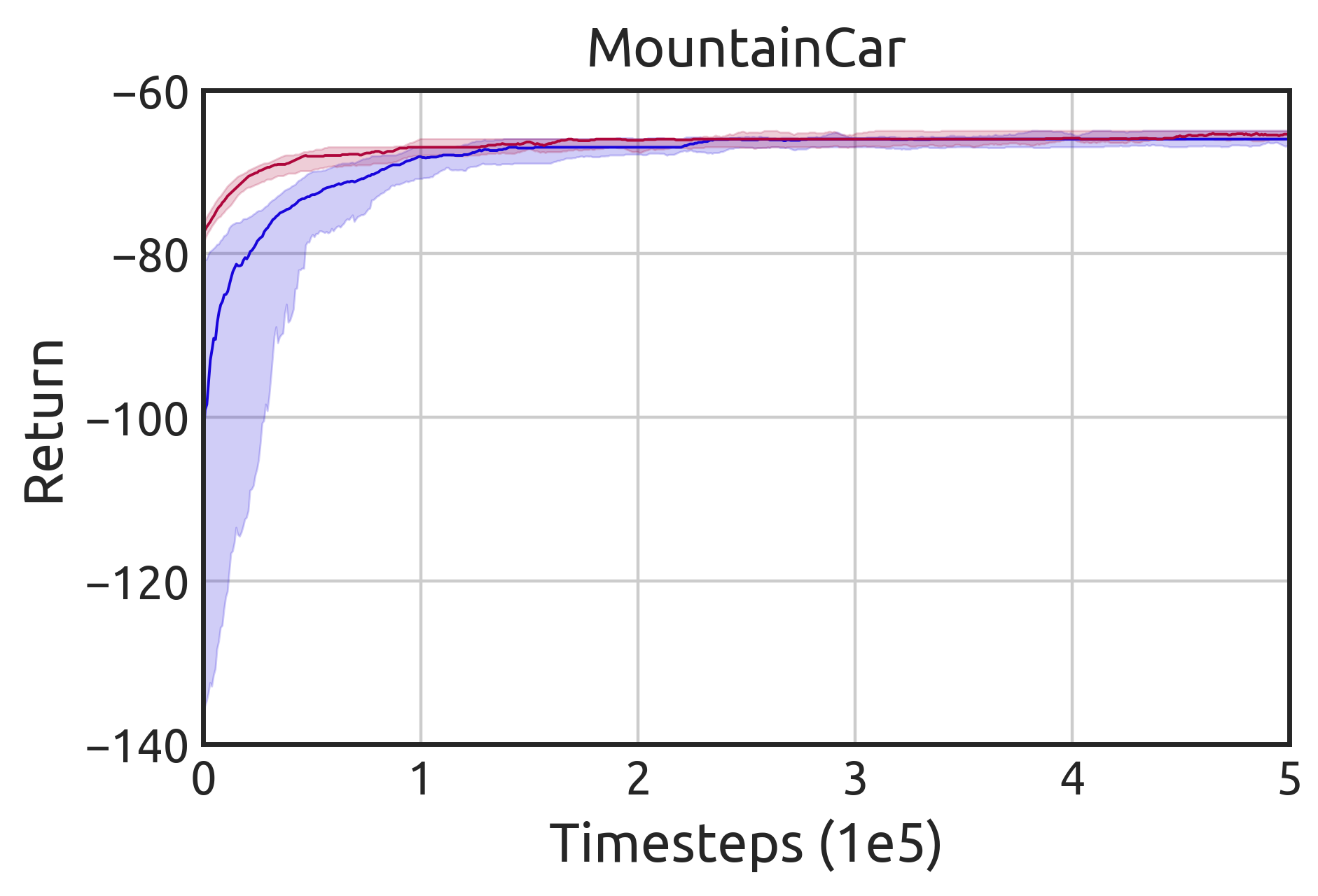}
\end{subfigure}%
\begin{subfigure}{.33\textwidth}
  \centering
  \includegraphics[width=\linewidth]{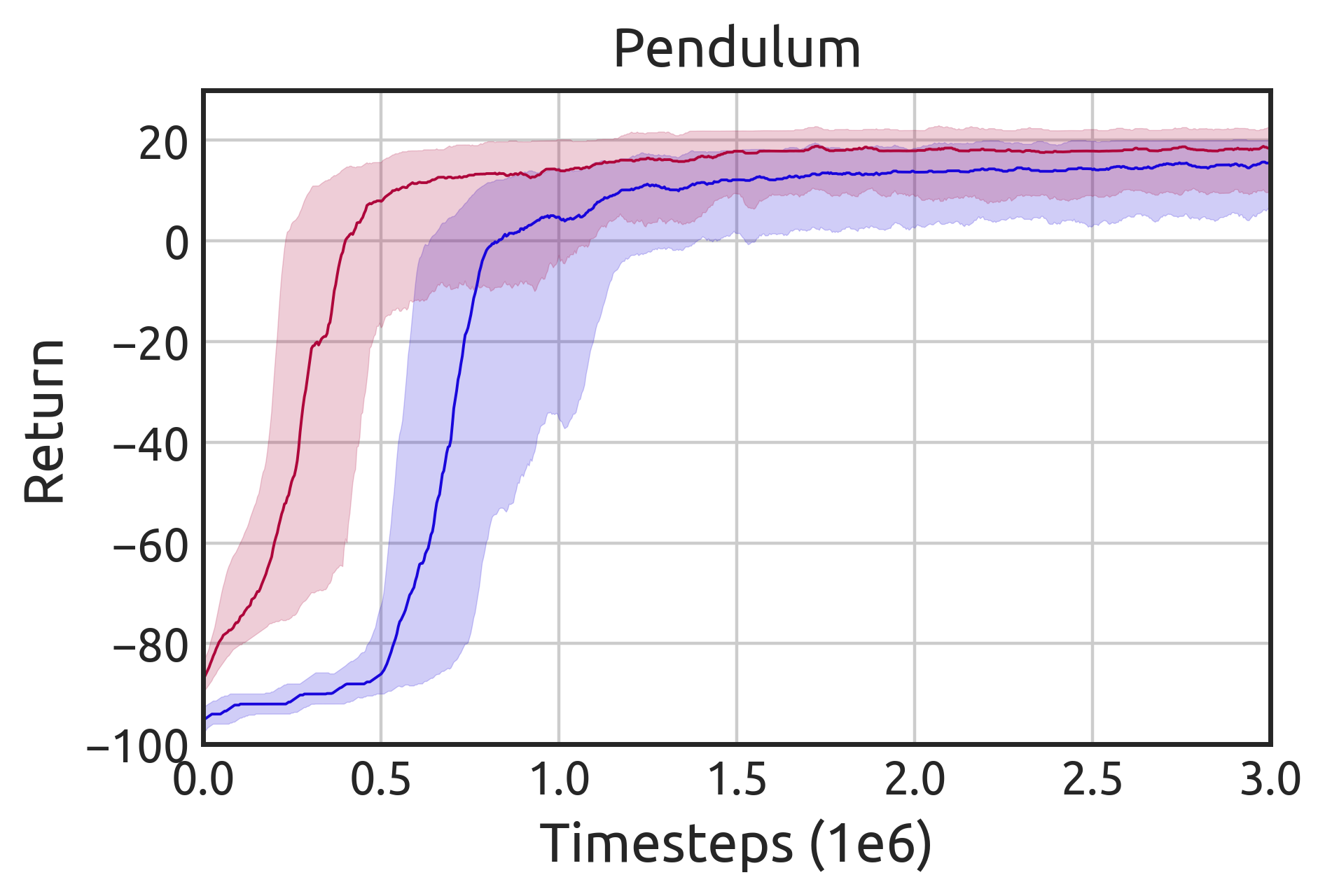}
\end{subfigure}
\begin{subfigure}{.33\textwidth}
  \centering
  \includegraphics[width=\linewidth]{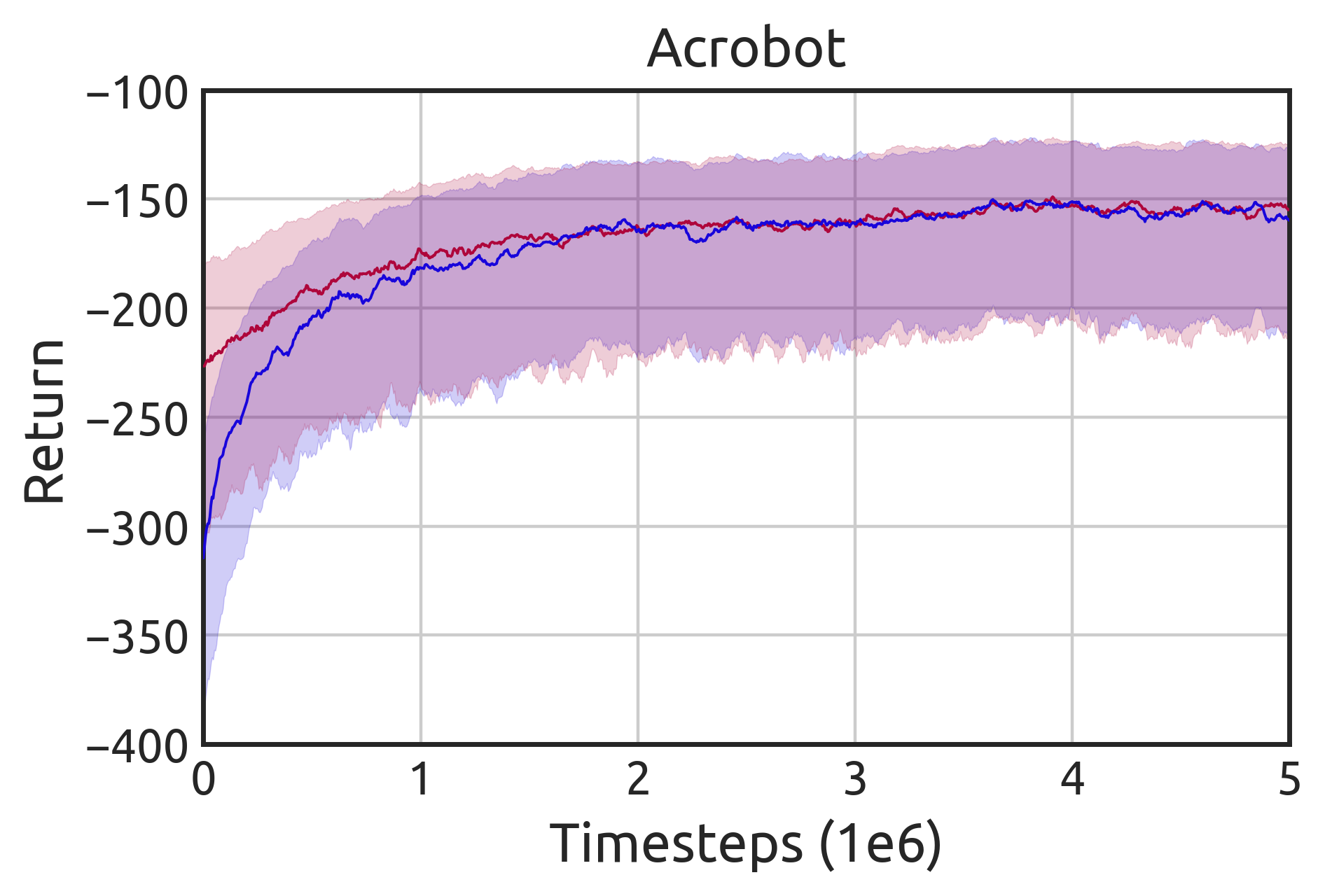}
\end{subfigure}\\
\begin{subfigure}{.4\textwidth}
  \centering
  \includegraphics[width=\linewidth]{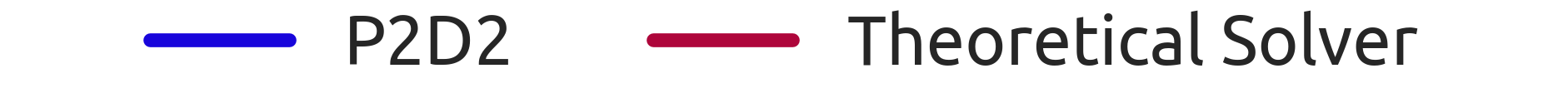}
\end{subfigure}%
\caption{Comparison of RL initialized with automatically discovered demonstrations (P2D2) and expert demonstrations (Theoretical Solver). Trendlines are medians and shaded areas are interquartile ranges aggregated over $10$ random seeds. Results show that degradation due to P2D2 demonstrations is small, and the P2D2-initialized agents catch up to the baseline agents before convergence occurs. }
\label{fig:demonstration_results}
\end{figure*}

\textbf{Comparison to classic and intrinsic exploration on RL benchmarks}.
% Performance measure
We compare P2D2 with state-of-the-art RL algorithms on several RL benchmarks. Performance is measured in terms of undiscounted returns and aggregated over 10 random seeds, sampled at random for each environment. We focus on domains with sparse rewards, which are notoriously difficult to explore for traditional RL methods (cf. Appendix~\ref{sec:sup:environments}).
Note that in all of the following experiments, P2D2 utilises basic RRT to discover demonstrations, and a simple LfD technique (supervised learning) is used. However, P2D2 is compatible with more sophisticated versions of RRT and LfD, which would most likely increase performance if used.
% Talk about TRPO and VIME
Our experiments focus on the widely-used methods TRPO~\cite{schulman2015trust} and DDPG~\cite{lillicrap2015continuous}. P2D2-TRPO and P2D2-DDPG are compared to the baseline algorithms with Gaussian action noise. As an additional baseline we include VIME-TRPO~\cite{houthooft2016vime}. VIME is an exploration strategy based on maximizing information gain about the agent's belief of the environment dynamics. It is included to show that P2D2 can improve on state-of-the-art exploration methods as well as naive ones, even though the return surface for VIME-TRPO is no longer flat, unlike Figure~\ref{fig:policy_initialization}.
The exact experimental setup is described in Appendix~\ref{sec:sup:experimental_setup}.

% How experiments were run.
P2D2 is first run to compute training trajectories for all environments.
The number of environment interactions during this phase is accounted for in the results, displayed as an offset with a vertical dashed black line. The average performance achieved by these trajectories is also reported as a guideline, with the exception of \emph{Cartpole Swingup} where doing so does not make sense. RL is then used to refine a policy pretrained with these trajectories.

Figure~\ref{fig:results} shows the median and interquartile range for all methods. P2D2 outdoes both vanilla and VIME baselines. It converges faster and achieves higher performance at the end of the experiment. In most cases, the upper quartile for our method begins well above the minimum return, indicating that P2D2 and pre-training are able to produce successful though not optimal policies. For the majority of problems, P2D2-DDPG performance starts significantly above the minimum return, plunges due to the inherent instability of DDPG, but eventually recovers, indicating that P2D2 pre-training can help mitigate the instability.

It is worth noting that P2D2's lower quartile is considerably higher than that of baselines. Indeed, for many of the baselines the lower quartile takes a long time to improve on the minimum return, and in some cases it never manages to do so at all. This is a common problem in sparse reward RL, where there is no alternative but to search the state space randomly until the first successful trajectory is found, and only then is an informative reward signal received.
While a few random seeds will by chance find a successful trajectory quickly (represented by the quickly rising upper quartile), others take a long time (represented by the much slower rise of the median and lower quartile). In other words, P2D2-TRPO/DDPG is much more robust to random policy initialization and to random seeds than standard RL methods. This is because P2D2 is able to use automatically discovered demonstrations to initialize policy parameters to a region with informative return gradients.
%Of particular note is the\emph{Hand Reach} task, where none of the baseline methods could find even a single successful trajectory. Although the $78$-dimensional state space is too large for traditional RL methods, P2D2 can search it efficiently to discover the necessary trajectories for learning.

\textbf{Comparison to learning from expert demonstrations}.
Since the demonstrations produced by P2D2 have no guarantee of optimality, it is reasonable to expect that a policy learned from them will be worse than a policy learned from expert demonstrations. This degradation is assessed in Figure~\ref{fig:demonstration_results}. Experiments were conducted on the \emph{MountainCar}, \emph{Pendulum} and \emph{Acrobot} tasks for which a theoretical solver is readily available (described in Appendix~\ref{sec:sup:theoretical_solvers}). The RL algorithm used was TRPO in all cases. As the figure shows, the P2D2 agent has worse initial performance than the baseline agent, but quickly catches up. The performance difference becomes statistically insignificant \emph{before} either agent converges. This contrasts with the previous set of experiments, where P2D2 agents converge long before other agents learning without demonstrations. The difference between learning from P2D2 demonstrations and learning from expert demonstrations is therefore very small when compared to the difference between learning from P2D2 and learning from scratch. These results suggest that in a LfD context, the quality of P2D2 demonstrations is relatively close to that of expert demonstrations.

\section{Conclusion}
\label{sec:conclusion}
We proposed Probabilistic Planning for Demonstration Discovery (P2D2), a paradigm for leveraging planing algorithms to automatically discover demonstrations, which can be converted to policies then refined by classic RL methods. This is a relaxation of the common assumption made in the LfD literature, which presupposes that either demonstrations or an expert are available. We provided theoretical guarantees of P2D2 finding solutions, as well as sampling complexity bounds. Empirical results show that P2D2 outperforms classic and intrinsic exploration techniques, requiring only a fraction of exploration samples and achieving better asymptotic performance. This holds true even when accounting for the cost of acquiring demonstrations.
Finally, experiments indicate that the degradation caused by using demonstrations from P2D2 rather than from an expert is mild, especially in comparison to the discrepancy between P2D2 and RL from scratch.
% On using goal information vs sparse rewards
%P2D2 exploration does not use rewards, instead it can use information about the goal to bias its search. This is very different from designing dense rewards, as this would require expert knowledge and/or information about dynamics, which P2D2 does not need.

% Future work
% Sampling assumption (need simulator)
In future work P2D2 could be extended to real-world problems by using sim-to-real methods, as mentioned in Section~\ref{sec:rrt}.
% high-d exploration is hard
Exploration in high-dimensional tasks is also challenging, as stated in Theorem~\ref{theorem:rrt_complex} and confirmed experimentally by increased P2D2 timesteps. Exploiting external prior knowledge and/or the structure of the problem can benefit exploration in high-dimensional tasks, and help make P2D2 practical for problems such as Atari games.
% RRT/LfD advances
Lastly, all the results in this paper were achieved using basic RRT and supervised LfD. Recent advances in RRT~\cite{chiang2019rl} and LfD~\cite{torabi2018behavioral} could therefore be applied to significantly improve P2D2.

\bibliographystyle{IEEEtran}
\bibliography{bibliography}

%%%%%%%%%% Merge with supplemental materials %%%%%%%%%%
\clearpage
\begin{center}
\textbf{\LARGE Learning from Demonstration without Demonstrations: Supplementary Material}
\end{center}
%%%%%%%%%% Merge with supplemental materials %%%%%%%%%%
%%%%%%%%%% Prefix a "S" to all equations, figures, tables and reset the counter %%%%%%%%%%
\setcounter{equation}{0}
\setcounter{figure}{0}
\setcounter{table}{0}
\setcounter{page}{1}
\setcounter{section}{0}
\makeatletter
\renewcommand{\theequation}{S\arabic{equation}}
\renewcommand{\thefigure}{S\arabic{figure}}
\renewcommand{\thesection}{S\arabic{section}}
\renewcommand{\thesubsection}{\thesection-\Alph{subsection}}

%\renewcommand{\thealgocf}{S\arabic{algocf}}
%%%%%%%%%% Prefix a "S" to all equations, figures, tables and reset the counter %%%%%%%%%%

\section{RRT algorithm psuedo-code}
\label{sec:sup:rrt}
In this section, we provide pseudo-code of the classic RRT algorithm. 

The basic form of RRT, used for path planning, attempts at every iteration to extend its tree $\mathbb{T}$ by adding a new vertex, which is biased by a randomly-selected state $s_{rand}$. The tree is expanded by selecting the nearest vertex of the tree, $s_{near}$,  to the newly sampled state  RRT then attempts to add a new vertex $s_{new}$ to the tree by applying an input $a$. This process biases the RRT to rapidly explore, which results in a uniform coverage of the state space \cite{Lavalle98rapidly-exploringrandom}. 

 Note that all trajectories represented by $\mathbb{T}$ are valid, as they are grown from valid transitions between states in $\mathcal{S}$. The robotics path planning community often adds explicit collision validity checks for $s_{rand}$ and $s_{new}$. However, the basic form of RRT~\cite{Lavalle98rapidly-exploringrandom}, does not require these checks, as expansion of the tree toward $s_{new}$ is inherently valid given transition dynamics. In RL problems, validity might extend beyond collision checks (e.g. not all random images are valid states of an Atari game), but the transition dynamics still guarantee $s_{new}$ will be valid.

\begin{algorithm}[h]
    \caption{Rapidly-exploring Random Trees (RRT)}
	\label{algo:rrt}
	\begin{algorithmic}
       \STATE {\bfseries Input:} $s_{init}$, $k$: sampling budget
       \STATE {\bfseries Output:} $\delta t$: Euler integration time interval

       \STATE $\mathbb{T}$.init($s_{init}$)
       \FOR{$i = 1:k$}
       \STATE $s_{rand} \leftarrow \text{RANDOM\_UNIFORM}(S)$
       %\IF{$s_{rand}$ is not valid}
       %\STATE pass
       %\ENDIF
	   \STATE $s_{near} \leftarrow \argmin_{s \in \mathbb{T}}\norm{s_{rand}-s_{near}}$ \COMMENT{Find nearest vertex}
	   \STATE $a \leftarrow \Upsilon(s_{near}, s_{rand})$ \COMMENT{Sample action}
	   \STATE $s_{new} \leftarrow s_{near}+\Delta t \cdot f(s_{near}, a)$ \COMMENT{Propagate to new state, Equation~(\ref{eq:diff_cont})}
	   %\IF{$\text{VALID\_TRANSITION}(s_{near},s_{new})$}
	   \STATE $\mathbb{T}$.add\_vertex($s_{new}$)
	   \STATE $\mathbb{T}$.add\_edge($s_{near},s_{new}$)
	   %\ENDIF
	   \ENDFOR
    \end{algorithmic}
% 	\DontPrintSemicolon
% 	\KwIn{$s_{init}$}
% 	\myinput{$k$: sampling budget}
% 	\myinput{$\delta t$: Euler integration time interval}
	
% 	\KwOut{$\mathbb{T}$}
    
%     $\mathbb{T}$.init($s_{init}$)\\
%     \For{$i = 1:k$}
% 	{
% 	    $s_{rand} \leftarrow \text{RANDOM\_UNIFORM}(S)$\\
	    
% 	    \If{$s_{rand} \notin \mathcal{F}$}{
% 	    pass\\}

% 	    $s_{near} \leftarrow \argmin_{s \in \mathbb{T}}\norm{s_{rand}-s_{near}}$ \tcc*{Find nearest vertex}

% 	     $a \leftarrow \Upsilon(s_{near}, s_{rand})$ \tcc*{Sample action}

% 	    $s_{new} \leftarrow s_{near}+\Delta t \cdot f(s_{near}, a)$ \tcc*{Propagate to new state, Eq.~(\ref{eq:diff_cont})}
	    
% 	    \If{$\text{VALID\_TRANSITION}(s_{near},s_{new})$}
% 	    {
% 	    $\mathbb{T}$.add\_vertex($s_{new}$)\\
% 	    $\mathbb{T}$.add\_edge($s_{near},s_{new}$) \\
% 	    }
% 	}
\end{algorithm} 

\section{Proof of Theorem~\ref{lemma:rrt}}\label{app:B}
This appendix proves Theorem 1, which shows that planning using RRT under differential constraints is probabilistically complete. The following proof is a modification of Theorem 2 from \cite{kleinbort2019}, where completeness of RRT in the RL setting is maintained without the need to explicitly sample a duration for every action.

Equation~\ref{eq:diff_cont} defines the environment's differential constraints. 
In practice, the equation is approximated by an Euler integration step, with the interval $[0,t_{\tau}]$ divided into $l >> n_{\tau}$ equal time intervals of duration $h$ with $t_{\tau}=l\cdot h$. The valid transition between consecutive time steps is given by:
\begin{equation}\label{eq:Euler}
    \begin{aligned}
    &s_{n+1} = s_{n}+f(s_n, a_n) \cdot h, \qquad s_{n},s_{n+1}\in \tau,\\
    &s.t. \lim\limits_{l \rightarrow \infty, h \rightarrow 0} \norm{s_n - \tau(n \cdot h)} = 0.
    \end{aligned}
\end{equation}
Furthermore, we define $\mathcal{B}_r(s)$ as a ball with a radius $r$ centered at $s$ for any given state $s \in \mathcal{S}$.

We assume that the planning environment is Lipschitz continuous in both state and action, constraining the rate of change of Equation~(\ref{eq:Euler}). Formally, there exist positive constants $K_s, K_a >0$, such that $\forall s_0,s_1 \in \mathcal{S}, a_0,a_1\in \mathcal{A}:$
\begin{align} 
\norm{f(s_0, a_0)-f(s_0, a_1)} &\leq K_a\norm{a_0-a_1}, \\ 
\norm{f(s_0, a_0)-f(s_1, a_0)} &\leq K_s\norm{s_0-s_1}.
\end{align}

\begin{lemma}\label{lemma:sup:traj_similarity}
Let $\tau$, $\tau'$ be trajectories where $s_0=\tau(0)$ and $s'_0=\tau'(0)$ such that $\norm{s_0-s'_0} \leq \delta_s$, with $\delta_s$ a positive constant. Suppose that for each trajectory a piecewise constant action is applied, so that $\Upsilon(t) =a$ and $\Upsilon'(t) =a'$ is fixed during a time period $T \geq 0$. Then 
$\norm{\tau(T)-\tau'(T)} \leq e^{K_sT}\delta_s + TK_a e^{K_sT}\norm{a-a'}$.
\end{lemma}
The proof for Lemma~\ref{lemma:sup:traj_similarity} is given in Lemma 2 of \cite{kleinbort2019}. Intuitively, this bound is derived from compounding worst-case divergence between $\tau$ and $\tau'$ at every Euler step along $T$ which leads to an overall exponential dependence. 

Using Lemma~\ref{lemma:sup:traj_similarity}, we want to provide a lower bound on the probability of choosing an action that will expand the tree successfully. We note that this scenario assumes that actions are drawn uniformly from $\mathcal{A}$, i.e. there is no steering function \footnote{The function $steer: \mathcal{S}\times\mathcal{S} \rightarrow \mathcal{A}$ returns an action $a_{steer}$ given two states $s{rand}$ and $s_{near}$ such that $a_{steer} = \argmin_{a \in \mathcal{A}}\norm{s_{rand}-(s_{near}+\Delta t \cdot f(s_{near}, a))} \: s.t. \: \norm{\Delta t \cdot f(s_{near}, a)} < \eta$, for a pre-specified $\eta > 0$ \cite{Karaman2011}.}. When better estimations of the steering function are available, the performance of RRT significantly improves. 

\begin{mydef}
A trajectory $\tau$ is defined as $\delta$-clear if $\exists\delta_{clear} > 0$ such that $\mathcal{B}_{\delta_{clear}}(\tau(t)) \in \mathcal{S}$  $\quad \forall t \in [0,t_\tau]$.
\end{mydef}

\begin{lemma}\label{lemma:sup:pick_action}
Suppose that $\tau$ is a valid trajectory from $\tau(0)=s_0$ to $\tau(t_\tau)=s_{goal}$ with a duration of $t_\tau$ and a clearance of $\delta$. Without loss of generality, we assume that actions are fixed for all $t\in [0,t_\tau]$, such that $\Upsilon(t)=a \in \mathcal{A}$.

Suppose that RRT expands the tree from a state $s'_{0} \in \mathcal{B}_{(\kappa\delta-\epsilon)}(s_0)$ to a state $s'_{goal}$, for any $\kappa \in (0,1]$ and $\epsilon \in (0,\kappa \delta)$ we can define the following bound:
\begin{equation*}
    \Pr[s'_{goal} \in \mathcal{B}_{\kappa \delta}(s_{goal}))] \geq \frac{\zeta_{\abs{\mathcal{S}}} \cdot \frac{\kappa\delta-e^{K_s t_\tau}(\kappa\delta-\epsilon)}{K_a t_\tau e^{K_s t_\tau}}}{\abs{\mathcal{A}}}.     
\end{equation*}
Here, $\zeta_{\abs{\mathcal{S}}} =\abs{\mathcal{B}_{1}(\cdot)}$ is the Lebesgue measure for a unit circle in $\mathcal{S}$.
\end{lemma}
\begin{proof}
We denote $\tau'$ a trajectory that starts from $s'_{0}$ and is expanded with an applied random action $a_{rand}$. According to Lemma~\ref{lemma:sup:traj_similarity}, 
\begin{align*}
    \norm{\tau(t)-\tau'(t)} &\leq e^{K_st}\delta_s + tK_a e^{K_st}\norm{a-a_{rand}}\\
    &\leq e^{K_st}(\kappa\delta-\epsilon) + tK_a e^{K_st}\norm{a-a_{rand}},
\end{align*}
for all $t \in [0,t_\tau]$, where $\delta_s \leq \kappa\delta-\epsilon$ since $s'_{0} \in \mathcal{B}_{\kappa\delta-\epsilon}(s_0)$. Now, we want to find $\norm{a-a_{rand}}$ such that the distance between the goal points of these trajectories, i.e. in the worst-case scenario, is bounded:
\begin{equation*}
e^{K_s t_\tau}(\kappa\delta-\epsilon) + t_\tau K_a e^{K_st_\tau}\norm{a-a_{rand}} < \kappa \delta.
\end{equation*}
After rearranging this formula, we can obtain a bound for $\norm{a-a_{rand}}$:
\begin{equation*}
 \Delta a = \norm{a-a_{rand}} < \frac{\kappa \delta - e^{K_s t_\tau}(\kappa\delta-\epsilon)}{t_\tau K_a e^{K_st_\tau}}.
\end{equation*}
Assuming that $a_{rand}$ is drawn out of a uniform distribution, the probability of choosing the proper action is 
\begin{equation}\label{eq:p_a_limit}
    p_a = \frac{\zeta_{\abs{\mathcal{S}}} \cdot \frac{\kappa \delta - e^{K_s t_\tau}(\kappa\delta-\epsilon)}{t_\tau K_a e^{K_st_\tau}}}{\abs{\mathcal{A}}},
\end{equation}
where $\zeta_{\abs{\mathcal{S}}}$ is used to account for the degeneracy in action selection due to the dimensionality of $\mathcal{S}$. We note that $\epsilon \in (0,\kappa \delta)$ guarantees $p_a \geq 0$, thus the probability is valid.
\end{proof}

Equation~\ref{eq:p_a_limit} provides a lower bound for the probability of choosing the suitable action. The following lemma provides a bound on the probability of randomly drawing a state that will expand the tree toward the goal.
\begin{lemma}\label{lemma:transition}
Let $s \in \mathcal{S}$ be a state with clearance $\delta$, i.e. $\mathcal{B}_{\delta}(s) \in \mathcal{S}$. Suppose that for an RRT $\mathbb{T}$ there exists a vertex $v \in \mathbb{T}$ such that $v \in \mathcal{B}_{2\delta/5}(s)$. Following the definition in Section~\ref{sec:rrt}, we denote $s_{near} \in \mathbb{T}$ as the closest vertex to $s_{rand}$. The probability that $s_{near} \in \mathcal{B}_{\delta}(s)$ is at least $\abs{\mathcal{B}_{\delta/5}(s)}/\abs{S}$.
\end{lemma}
\begin{proof}
Let $s_{rand} \in \mathcal{B}_{\delta/5}(s)$. Therefore the distance between $s_{rand}$ and $v$ is upper-bounded by $\norm{s_{rand}-v} \leq  3\delta/5$. If there exists a vertex $s_{near}$ such that $s_{near} \neq v$ and $\norm{s_{rand}-s_{near}} \leq \norm{s_{rand}-v}$, then $s_{near} \in \mathcal{B}_{3\delta/5}(s_{rand}) \subset	\mathcal{B}_{\delta}(s)$. Hence, by choosing $s_{rand} \in \mathcal{B}_{\delta/5}(s)$, we are guaranteed $s_{near} \in \mathcal{B}_{\delta}(s)$. As $s_{rand}$ is drawn uniformly, the probability for $s_{rand} \in \mathcal{B}_{\delta/5}(s)$ is at least $\abs{\mathcal{B}_{\delta/5}(s)}/\abs{S}$.    
\end{proof}

We can now prove the main theorem.
\begin{customthm}{1}\label{theorem:sup:rrt}
	 Suppose that there exists a valid trajectory $\tau$ from $s_0$ to $\mathcal{S}_{goal}$ as defined in definition~\ref{def:traj}, with a corresponding piecewise constant control. The probability that RRT fails to reach $\mathcal{S}_{goal}$ from $s_0$ after $k$ iterations is bounded by $a e^{-bk}$, for some constants $a,b > 0$.
\end{customthm}
\begin{proof}
Lemma~\ref{lemma:sup:pick_action} puts a bound on the probability to find actions that expand the tree from one state to another in a given time. As we assume that a valid trajectory exists, we can assume that the probability defined in Lemma~\ref{lemma:sup:pick_action} is non-zero, i.e. $p_a > 0$, hence:
\begin{equation}\label{eq:time_limit}
    \kappa \delta - e^{K_s \Delta t}(\kappa\delta-\epsilon) > 0,
\end{equation}
where we set $\kappa = 2/5$ and $\epsilon = 5^{-2}$ as was also done in \cite{kleinbort2019}. We additionally require that $\Delta t$, which is typically defined as an RL environment parameter, is chosen accordingly so to ensure that Equation~(\ref{eq:time_limit}) holds, i.e. $K_s \Delta t < \log\left(\frac{\kappa \delta}{\kappa \delta - \epsilon}\right)$.

We cover $\tau$ with balls of radius $\delta = \min\{\delta_{goal}, \delta_{clear}\}$, where $\mathcal{B}_{\delta_{goal}} \subseteq \mathcal{S}_{goal}$. The balls are spaced equally in time with the center of the $i^{th}$ ball is in $c_i = \tau(\Delta t \cdot i), \forall i=0:m$, where $m=t_\tau/\Delta t$. Therefore, $c_0= s_0$ and $c_m=s_{goal}$. We now examine the probability of RRT propagating along $\tau$. Suppose that there exists a vertex $v \in \mathcal{B}_{2\delta/5}(c_i)$. We need to bound the probability $p$ that, by taking a random sample $s_{rand}$, there will be a vertex $s_{near} \in \mathcal{B}_{\delta}(c_i)$ such that $s_{new} \in \mathcal{B}_{2\delta/5}(c_{i+1})$. Lemma~\ref{lemma:transition} provides a lower bound for the probability that $s_{near} \in \mathcal{B}_{\delta}(c_i)$, given that there exists a vertex $v \in \mathcal{B}_{2\delta/5}(c_i)$. The bound is $\abs{\mathcal{B}_{\delta/5}(s)}/\abs{S}$. Lemma~\ref{lemma:sup:pick_action} provide a lower bound for the probability of choosing an an action from $s_{near}$ to $s_{new}$. This bound is $\rho \equiv  \frac{\zeta_{\abs{\mathcal{S}}} \cdot \frac{\kappa \delta - e^{K_s \Delta t}(\kappa\delta-\epsilon)}{\Delta t K_a e^{K_s \Delta t}}}{\abs{\mathcal{A}}} > 0$, where we  have substituted $t_\tau$ with $\Delta t$. Consequently, $p \geq (\abs{\mathcal{B}_{\delta/5}(s)}\cdot \rho)/\abs{S}$.

For RRT to recover $\tau$, the transition between consecutive circles must be repeated $m$ times. This stochastic process can be described as a binomial distribution, where we perform $k$ trials (randomly choosing $s_{rand}$), with $m$ successes (transition between circles) and a transition success probability $p$. The probability mass function of a binomial distribution is 
$\Pr(X=m)= \Pr(m;k,p) = \binom{k}{m}p^{m}(1-p)^{k-m}$.
We use the cumulative distribution function (CDF) to represent the upper bound for failure, i.e. the process was unable to perform $m$ steps, which can be expressed as:
\begin{equation}
    \Pr(X < m)=\sum_{i=0}^{m-1}\binom{k}{i}p^{i}(1-p)^{k-i}.
\end{equation}
Using Chernoff's inequality we derive the tail bounds of the CDF when $m \leq p\cdot k$:
\begin{align}
    \Pr(X < m) &\leq \exp \left(-\frac {1}{2p}\frac{(kp-m)^{2}}{k}\right)\\
    &=\exp \left(-\frac {1}{2}kp+m-\frac{m^{2}}{kp}\right)\\
    &\leq e^{m}e^{-\frac {1}{2}pk} = ae^{-bk}. \label{eq:greta}
\end{align}
In the other case, where $p < m/k < 1$, the upper bound is given by \cite{Arratia1989}: 
\begin{align}\label{eq:bound2}
    \Pr(X < m) &\leq \exp \left(-k\mathcal{D}\left({\frac {m}{k}}\parallel p\right)\right),
\end{align}
where $\mathcal{D}$ is the relative entropy such that 
\begin{equation*}
    D\left({\frac {m}{k}}\parallel p\right) = \frac{m}{k}\log {\frac {\frac {m}{k}}{p}}+(1-\frac {m}{k})\log {\frac {1-\frac {m}{k}}{1-p}}.
\end{equation*}
Rearranging $\mathcal{D}$, we can rewrite \ref{eq:bound2} as follows:
\begin{align}
    &\Pr(X < m) \leq\\
    &\exp \left(-k \left({\frac{m}{k}}\log\left(\frac{m}{kp}\right)+\frac{k-m}{k}\log\left({\frac {1-\frac{m}{k}}{1-p}}\right)\right)\right)\\
    &= \exp\left(-m\log\left(\frac{m}{kp}\right)\right)\exp\left(-k\log\left(\frac {1-\frac{m}{k}}{1-p}\right)\right)\\
    &\qquad*\exp\left(m\log\left(\frac {1-\frac{m}{k}}{1-p}\right)\right)\\
    &=\exp\left(-m\log\left(\frac{m(1-p)}{kp(1-\frac{m}{k})}\right)\right)\\
    &\qquad*\exp\left(-k\log\left(\frac {1-\frac{m}{k}}{1-p}\right)\right) \\
    & \leq \exp\left(-k\log\left(\frac {1-\frac{m}{k}}{1-p}\right)\right) \label{eq:jenny}\\
    & \leq \exp\left(-k\log\left(\frac {0.5}{1-p}\right)\right) \label{eq:david}\\
    & \leq e^{-kp} = ae^{-bk}\label{eq:zupko},
\end{align}
where (\ref{eq:jenny}) is justified for worst-case scenario where $p=m/k$, (\ref{eq:david}) uses the fact that $p < m/k < 0.5$, hence $ 1- m/k > 0.5$. The last step, (\ref{eq:zupko}) is derived from the first term of the Taylor expansion of $\log\left(\frac{1}{1-p}\right) = \sum_{j=1}^{\infty} \frac{p^j}{j}$.

As $p$ and $m$ are fixed and independent of $k$, we show that the expression for $\Pr(X <m)$ decays to zero exponentially with $k$, therefore RRT is probabilistically complete. 
\end{proof}

It is worth noting that the failure probability $\Pr(X <m)$ depends on problem-specific properties, which give rise to the values of $a$ and $b$. Intuitively, $a$ depends on the scale of the problem such as volume of the goal set $\abs{\mathcal{S}^{RL}_{goal}}$ and how complex and long the solution needs to be, as evident in Equation~(\ref{eq:greta}). More importantly, $b$ depends on the probability $p$. Therefore, it is a function of the dimensionality of $\mathcal{S}$ (through the probability of sampling $s_{rand}$) and other environment parameters such as clearance (defined by $\delta$) and dynamics (via $K_s$, $K_a$), as specified in Equation~(\ref{eq:p_a_limit}).

%We assume that the planning environment is Lipschitz continuous for both state and action, constraining the rate of change of Equation~\ref{eq:Euler}. F update 
%where  method, where actions taken along the trajectory are defined as piecewise constant, similar to ~\cite{kleinbort2019}.

\section{Experimental setup}
\label{sec:sup:experimental_setup}
% A description of the computing infrastructure used.
All experiments were run using a single $2.2$GHz core and a GeForce GTX 1080 Ti GPU.

\subsection{Environments}
\label{sec:sup:environments}
All environments are made available in supplementary code.
Environments are based on Gym~\cite{gym}, with modified sparse reward functions and state spaces. All environments emit a $-1$ reward per timestep unless noted otherwise. The environments have been further changed from Gym as follows:
\begin{itemize}
	\item \emph{Cartpole Swingup}-
    The state space $\mathcal{S} \subseteq \mathbb{R}^4$ consists of states $s=\left[x, \theta, \dot{x}, \dot{\theta}\right]$ where $x$ is cart position, $\dot{x}$ is cart linear velocity, $\theta$ is pole angle (measuring from the $y$-axis) and $\dot{\theta}$ pole angular velocity. Actions $\mathcal{A} \subseteq \mathbb{R}$ are force applied on the cart along the $x$-axis. The goal space $\mathcal{S}_{goal}$ is $\{s \in \mathcal{S} \mid \cos{\theta} > 0.9 \}$. Note that reaching the goal space does not terminate an episode, but yields a reward of $\cos{\theta}$. Time horizon is $H=500$.  Reaching the bounds of the rail does not cause failure but arrests the linear movement of the cart.

	\item \emph{MountainCar}- 
    The state space $\mathcal{S} \subseteq \mathbb{R}^2$ consists of states $s=\left[x, \dot{x}\right]$ where $x$ is car position and  $\dot{x}$ is car velocity. Actions $\mathcal{A} \subseteq \mathbb{R}$ are force applied by the car engine.  The goal space $\mathcal{S}_{goal}$ is $\{s \in \mathcal{S} \mid x \ge 0.45 \}$. Time horizon is $H=200$.

	\item \emph{Acrobot}-
    The state space $\mathcal{S} \subseteq \mathbb{R}^4$ consists of states $s=\left[\theta_0, \theta_1, \dot{\theta_0}, \dot{\theta_1}\right]$ where $\theta_0, \theta_1$ are the angles of the joints (measuring from the $y$-axis and from the vector parallel to the $1^{st}$ link, respectively) and $\dot{\theta_0}, \dot{\theta_1}$ are their angular velocities. Actions $\mathcal{A} \subseteq \mathbb{R}$ are torque applied on the $2^{nd}$ joint. The goal space $\mathcal{S}_{goal}$ is $\{s \in \mathcal{S} \mid -\cos{\theta_0} - \cos{(\theta_0 + \theta_1)} > 1.9\}$. In other words, the set of states where the end of the second link is at a height $y > 1.9$. Time horizon is $H=500$.

	\item \emph{Pendulum}-
	The state space $\mathcal{S} \subseteq \mathbb{R}^2$ consists of states $s=\left[\theta, \dot{\theta}\right]$ where $\theta$ is the joint angle (measured from the $y$-axis) and $\dot{\theta}$ is the joint angular velocity. Actions $\mathcal{A} \subseteq \mathbb{R}$ are torque applied on the joint. The goal space $\mathcal{S}_{goal}$ is $\{s \in \mathcal{S} \mid \cos{\theta} > 0.99 \}$. Note that reaching the goal space does not terminate an episode, but yields a reward of $\cos{\theta}$. Time horizon is $H=100$.

	\item \emph{Reacher}-
    The state space $\mathcal{S} \subseteq \mathbb{R}^6$ consists of states $s=\left[\theta_0, \theta_1, x, y \dot{\theta_0}, \dot{\theta_1}\right]$ where $\theta_0, \theta_1$ are the angles of the joints, $(x,y)$ are the coordinates of the target and $\dot{\theta_0}, \dot{\theta_1}$ are the joint angular velocities. Actions $\mathcal{A} \subseteq \mathbb{R}^2$ are torques applied at the $2$ joints. The goal space $\mathcal{S}_{goal}$ is the set of states where the end-effector is within a distance of $0.01$ from the target. Time horizon is $H=50$.
    
    \item \emph{Fetch Reach}-
    A high-dimensional robotic task where the state space $\mathcal{S} \subseteq \mathbb{R}^{13}$ consists of states
    $s=[gripper\_pos, finger\_pos, gripper\_state,$ $finger\_state, goal\_pos]$ where the Cartesian coordinates and velocities of the Fetch robot's gripper are $gripper\_pos$ and $gripper\_vel$, and $finger\_state$ and $finger\_vel$ are the two-dimensional position and velocity of the gripper fingers, and $goal\_pos$ are the Cartesian coordinates of the goal. Actions $\mathcal{A} \subseteq \mathbb{R}^4$ are relative target positions of the gripper and fingers, which the MuJoCo controller will try to achieve. The goal space $\mathcal{S}_{goal}$ is the set of states where the end-effector is within a distance of $0.05$ from $goal\_pos$. Time horizon is $H=50$.
    \linebreak Note that this problem is harder than the original version in OpenAI Gym, as we only sample $goal\_pos$ that are far from the gripper's initial position.
    
    \item \emph{Hand Reach}-
    A high-dimensional robotic task where the state space $\mathcal{S} \subseteq \mathbb{R}^{78}$ consists of states $s=\left[joint\_pos, joint\_vel, fingertip\_pos, goal\_pos\right]$ where $joint\_pos, joint\_vel$ are the angles and angular velocities of the Shadow hand's $24$ joints, $fingertip\_pos$ are the Cartesian coordinates of the $5$ fingertips, and $goal\_pos$ are the Cartesian coordinates of the goal positions for each fingertip. Actions $\mathcal{A} \subseteq \mathbb{R}^{20}$ are absolute target angles of the $20$ controllable joints, which the MuJoCo controller will try to achieve. The goal space $\mathcal{S}_{goal}$ is the set of states where all fingertips are simultaneously within a distance of $0.02$ from their respective goals. Time horizon is $H=50$.
\end{itemize}

\subsection{Theoretical solvers}
\label{sec:sup:theoretical_solvers}
LfD algorithms typically use demonstrations from a human demonstrator, a previously learned policy, or a theoretical solver. In order to generate expert demonstrations of this type for Figure~\ref{fig:demonstration_results}, we used the following solvers:

\begin{itemize}
    \item \emph{MountainCar}-
    The solver is simply the policy
    \begin{equation}
        a = \sgn(\dot{x}).
    \end{equation}
    This will pump the maximum possible energy into the system at every time step. Although this solver overshoots slightly, it is guaranteed to solve the problem and is in practice very close to optimal.
    
    \item \emph{Pendulum}-
    The solver is the policy
    \begin{equation}
        a = \sgn(\dot{\theta}) \cdot e_{goal}(s) - e_{total}(s),
    \end{equation}
    where $e_{total}$ is the total energy of the system in state $s$ and $e_{goal}$ is the total energy of the system in the goal state $(0, 0)$.
    This policy is able to arrive at the upright position with low velocity. Note that the P2D2 algorithm does not have velocity information in its goal space, so this comparison unfairly favours the theoretical solver.
    
    \item \emph{Acrobot}-
    We used the solver described by~\cite{spong1994swing}.
\end{itemize}

\subsection{Experimental setup and hyper-parameter choices}
\label{sec:sup:exp_setup_hyper_param}
All experiments feature a policy with $2$ fully-connected hidden layers of $32$ units each with tanh activation, with the exception of \emph{Reacher}, for which a policy network of $4$ fully-connected hidden layers of $128$ units each with relu activation is used. For all environments we use a linear feature baseline for TRPO.

Default values are used for most hyperparameters.
A discount factor of $\gamma=0.99$ is used in all environments.
For VIME, hyperparameters values reported in the original paper are used, and the implementation published by the authors was used.

For TRPO, default hyperparameter values and implementation from Garage are used: KL divergence constraint $\delta=10^{-2}$, and Gaussian action noise $\mathcal{N}(0,0.3^2)$.

In comparisons with VIME-TRPO and vanilla TRPO, the P2D2 goal sampling probability $p_g$ is set to $0.05$, as proposed in~\cite{Urmson2003}.
Goal sets $\mathcal{S}_{goal}$ are defined in Appendix~\ref{sec:sup:environments} for each environment. In all experiments, the local policy $\pi_l$ learned by P2D2 is an approximate Gaussian process, combining Bayesian linear regression with prior precision $0.1$ and noise precision $1.0$, with $300$ random Fourier features~\cite{rahimi2008random} approximating a square exponential kernel with lengthscale $0.3$.

\end{document}